\documentclass[runningheads]{llncs}
\usepackage{graphicx}
\usepackage{tikz}
\usepackage{comment}
\usepackage{amsmath,amssymb} % define this before the line numbering.
\usepackage{color}
\usepackage{graphicx}
\usepackage{subfigure}
\usepackage{caption}
\usepackage{bm}
\usepackage{wrapfig}
\usepackage{enumitem}
\usepackage{breakcites}
\usepackage{setspace}

\usepackage{booktabs} 
\usepackage{makecell}
\usepackage{slashbox}
\usepackage{pifont}
\usepackage{multirow}
\usepackage[accsupp]{axessibility}  % Improves PDF readability for those with disabilities.

\usepackage[pagebackref=true,breaklinks=true,colorlinks,bookmarks=false]{hyperref}
\hypersetup{linkcolor=[rgb]{0.7,0.1,0.1}}
\hypersetup{citecolor=[rgb]{0.4,0.15,0.95}}
\usepackage{xcolor,colortbl}
\definecolor{Gray}{gray}{0.94}

\newcommand{\eg}{\emph{e.g.}}
\newcommand{\ie}{\emph{i.e.}}

\renewcommand{\captionlabelfont}{\footnotesize}
\newlength\savewidth\newcommand\shline{\noalign{\global\savewidth\arrayrulewidth
  \global\arrayrulewidth 1pt}\hline\noalign{\global\arrayrulewidth\savewidth}}

\usepackage[width=122mm,left=12mm,paperwidth=146mm,height=193mm,top=12mm,paperheight=217mm]{geometry}

\begin{document}

% \linenumbers
\pagestyle{headings}
\mainmatter
\def\ECCVSubNumber{656}  % Insert your submission number here

\title{Structural Causal 3D Reconstruction} % Replace with your title



\titlerunning{Structural Causal 3D Reconstruction}
\author{\small Weiyang Liu\textsuperscript{1,2*} \and
Zhen Liu\textsuperscript{3*} \and
Liam Paull\textsuperscript{3} \and Adrian Weller\textsuperscript{2,4} \and Bernhard Sch\"olkopf\textsuperscript{1}}
\authorrunning{W. Liu et al.}
\institute{\textsuperscript{1}Max Planck Institute for Intelligent Systems, T\"ubingen\\ \textsuperscript{2}University of Cambridge\ \ \ \textsuperscript{3}Mila, Universit\'e de Montr\'eal\ \ \  \textsuperscript{4}Alan Turing Institute}

%******************
\maketitle

\begin{abstract}
This paper considers the problem of unsupervised 3D object reconstruction from in-the-wild single-view images. Due to ambiguity and intrinsic ill-posedness, this problem is inherently difficult to solve and therefore requires strong regularization to achieve disentanglement of different
latent factors. Unlike existing works that introduce explicit regularizations into objective functions, we look into a different space for implicit regularization -- the structure of latent space. Specifically, we restrict the structure of latent space to capture a topological causal ordering of latent factors (\ie, representing causal dependency as a directed acyclic graph). We first show that different causal orderings matter for 3D reconstruction, and then explore several approaches to find a task-dependent causal factor ordering. Our experiments demonstrate that the latent space structure indeed serves as an implicit regularization and introduces an inductive bias beneficial for reconstruction.
\end{abstract}

\vspace{-5mm}
\section{Introduction}
\vspace{-1mm}
Understanding the 3D structures of objects from their 2D views has been a longstanding and fundamental problem in computer vision. 
Due to the lack of high-quality 3D data, unsupervised single-view 3D reconstruction is typically favorable; however, it is an ill-posed problem by nature, and it typically requires a number of carefully-designed priors and regularizations to achieve good disentanglement of latent factors~\cite{kanazawa2018learning,suwajanakorn2018discovery,chen2019learning,kato2019learning,chen2019unsupervised,wu2020unsupervised,feng2021learning}. Distinct from these existing works that focus on introducing explicit regularizations, we aim to explore how the structure of latent space can implicitly regularize 3D reconstruction, and to answer the following question: \emph{Can a suitable structure of latent space encode helpful implicit regularization and yield better inductive bias?}

\begin{figure}[t]
  \renewcommand{\captionlabelfont}{\footnotesize}
  \setlength{\abovecaptionskip}{6pt}
  \setlength{\belowcaptionskip}{-10pt}
  \centering
  \vspace{-1mm}
  \includegraphics[width=4.65in]{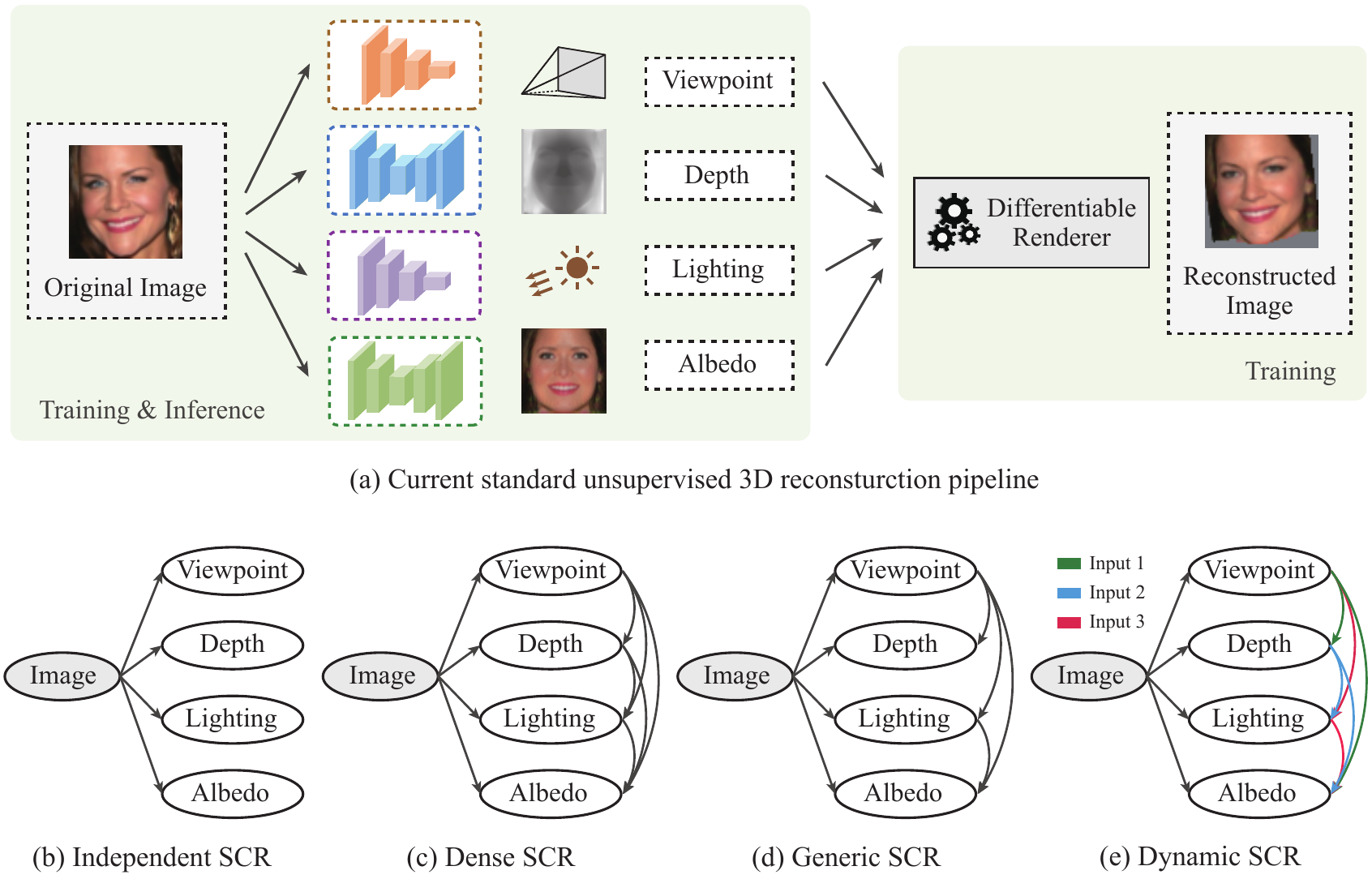}
  \caption{\footnotesize (a) Overview of the standard 3D reconstruction pipeline. (b) Graphical model of independent SCR, which is adopted in the standard pipeline and assumes full conditional independence. (c) Graphical model of a dense SCR example. This makes no assumption on the distribution. (d) Graphical model of a generic SCR example. This assumes partial conditional independence. (e) Graphical model of a dynamic SCR example. This yields strong flexibility. All directed edges are learned in practice.}\label{overview}
\end{figure}

Current single-view 3D reconstruction methods~\cite{wu2020unsupervised,kato2019learning,li2020self} typically decompose 3D objects into several semantic latent factors such as 3D shape, texture, lighting and viewpoint. These latent factors are independently extracted from single 2D images and then fed into a differentiable renderer to reconstruct the original 2D images, as illustrated in Fig.~\ref{overview}(a). Conditioned on the input image, these latent factors are typically assumed to be independent from each other. Such an assumption for disentanglement can be too strong and sometimes unrealistic, because it suggests that the estimated viewpoint will not affect the estimation of lighting in the image, which contradicts the formation of realistic images. This observation motivates us to explore how the dependency structure of latent factors implicitly regularizes the encoder and improves disentanglement. 

Taking inspiration from structural causal models~\cite{pearl2009causality}, we propose the \textbf{Structural Causal Reconstruction (SCR)} framework which introduces structural priors to the latent space. We consider the causal ordering of latent factors and study how different causal orderings can introduce different inductive biases.

Depending on the type of causal orderings and the corresponding flexibility, we derive three SCR variants: dense SCR which learns a chain factorization without any embedded conditional independence, generic SCR which learns a directed acyclic graph (DAG) over the latent factors, and dynamic SCR which learns a dynamic DAG that is dependent on the input image. We note that the standard 3D reconstruction pipeline can be viewed as independent SCR as shown in Fig.~\ref{overview}(b) (\ie, viewpoint, depth, lighting and albedo are conditionally independent from each other given the input image), while dense SCR does not assume any conditional independence. Generic SCR learns a DAG over the latent factors and serves as an interpolation between independence SCR and dense SCR by incorporating partial conditional independence. Both dense SCR and generic SCR are learned with a static ordering which is fixed once trained. To accommodate the over-simplified rendering model and the complex nature of image formation, we propose dynamic SCR that can capture more complex dependency by learning input-dependent DAGs. This can be useful when modeling in-the-wild images that are drawn from a complex multi-modal distribution~\cite{murphy2002dynamic}. Specifically, we apply Bayesian optimization to dense SCR to search for the best dense causal ordering of the latent factors. For generic SCR, we first propose to directly learn a DAG with an additional regularization. Besides that, we further propose a two-phase algorithm: first running dense SCR to obtain a dense ordering and then learning the edges via masking. For dynamic SCR, we propose a self-attention approach to learn input-dependent DAGs.

From a distribution perspective, independent SCR (Fig.~\ref{overview}(b)) is the least expressive graphical model in the sense that it imposes strong conditional independence constraints and therefore limits potential distributions that can factorize over it. On the contrary, any conditional distribution $P(\bm{V},\bm{D},\bm{L},\bm{A}|\bm{I})$ (where $\bm{V},\bm{D},\bm{L},\bm{A},\bm{I}$ denote viewpoint, depth, lighting, albedo and image, respectively) can factorize over dense SCR, making it the most expressive variant for representing distributions. Generic SCR unifies both independent SCR and dense SCR by incorporating a flexible amount of conditional independence constraints. Dynamic SCR is able to capture even more complex conditional distribution that is dynamically changing for different input images.

\setlength{\columnsep}{13pt}
\begin{wrapfigure}{r}{0.295\textwidth}
  \begin{center}
  \advance\leftskip+1mm
  \renewcommand{\captionlabelfont}{\footnotesize}
    \vspace{-0.57in}  
    \includegraphics[width=0.268\textwidth]{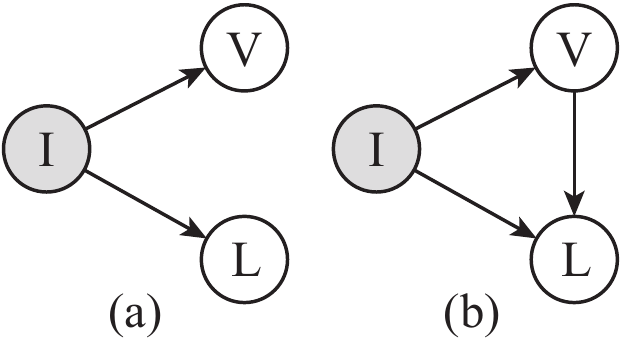}
    \vspace{-0.08in} 
    \caption{\footnotesize (a) Viewpoint and lighting are extracted independently from the input image. (b) The extracted viewpoint gives constraints on lighting. These arrows denote encoding latent variables from the image (\ie, anti-causal direction).
    }\label{intro-example}
    \vspace{-0.4in} 
  \end{center}
\end{wrapfigure}

Intuition for why learning a latent dependency structure helps 3D reconstruction comes from the underlying entanglement among estimated viewpoint, depth, lighting and albedo. For example, conditioned on a given 2D image, a complete disentanglement between viewpoint and lighting indicates that changing the estimated viewpoint of an object will not change its estimated lighting. This makes little sense, since changing the viewpoint will inevitably affect the estimation of lighting.
In contrast to existing pipelines that extract viewpoint and depth independently from the image (Fig.~\ref{intro-example}(a)), the information of viewpoint may give constraints on lighting (\eg, modeled as a directed edge from $\bm{V}$ to $\bm{L}$ in Fig.~\ref{intro-example}(b)). Therefore, instead of ignoring the natural coupling among latent factors and assuming conditional independence, we argue that learning a suitable dependency structure for latent factors is crucial for intrinsic disentanglement. In general, modeling the latent dependency and causality among viewpoint, depth, lighting and albedo renders an implicit regularization for disentanglement, leading to strong generalizability. Beside the intuition from the anti-causal direction, Section~\ref{causal} gives another interpretation for SCR from the causal direction. Our contributions are:

\vspace{-1mm}
\begin{itemize}
    \item We explicitly model the causal structure among the latent factors.
    \item To learn a causal ordering, we propose three SCR variants including dense SCR, generic SCR and dynamic SCR. Each one yields a different level of distribution expressiveness and modeling flexibility.
    \item We constrain the latent space structure to be a topological causal ordering (which can represent arbitrary DAGs), reducing the difficulty of learning.
    \item Our method is in parallel to most current 3D reconstruction pipelines and can be used simultaneously with different pipelines such as \cite{wu2020unsupervised,tewari2017mofa,feng2021learning,li2020self}.
    \item Our empirical results show that different causal orderings of latent factors lead to significantly different 3D reconstruction performance.
\end{itemize}

\section{Related Work}
\vspace{-0.75mm}

\textbf{Multi-view 3D reconstruction}. This method usually requires multi-view images of the same target object. Classical techniques such as Structure from Motion~\cite{ozyesil2017survey} and Simultaneous Localization and Mapping~\cite{fuentes2015visual} rely on hand-crafted geometric features and matching across different views. Owing to the availability of large 3D object datasets, modern approaches~\cite{choy20163d,kar2017learning,xie2019pix2vox} can perform multi-view 3D reconstruction with neural networks that map 2D images to 3D volumes.

\vspace{0.9mm}

\noindent\textbf{Shape from X}. There are many alternative monocular cues that can be used for reconstructing shapes from images, such as shading~\cite{horn1989shape,zhang1999shape}, silhouettes~\cite{koenderink1984does}, texture~\cite{witkin1981recovering} and symmetry~\cite{mukherjee1995shape,franccois2003mirror}. These methods are generally not applicable to in-the-wild images due to their strong assumptions. Shape-from-symmetry~\cite{mukherjee1995shape,franccois2003mirror,thrun2005shape,sinha2012detecting} assumes the symmetry of the target object, making use of the original image and its horizontally flipped version as a stereo pair for 3D reconstruction. \cite{sinha2012detecting} demonstrates the possibility to detect symmetries and correspondences using descriptors. Shape-from-shading assumes a specific shading model (\eg, Phong shading~\cite{phong1975illumination} and spherical harmonic lighting~\cite{green2003spherical}), and solves an inverse rendering problem to decompose different intrinsic factors from 2D images.

\vspace{0.9mm}

\noindent\textbf{Single-view 3D reconstruction}. This line of research \cite{choy20163d,girdhar2016learning,gwak2017weakly,tulsiani2017multi,wiles2017silnet,yan2016perspective,zhu2017rethinking,fan2017point,henderson2018learning,li2020self,wu2020unsupervised,hu2021self,fahim2021single} aims to reconstruct a 3D shape from a single-view image. \cite{wang2018pixel2mesh,pan2019deep,wen2019pixel2mesh++} use images and their corresponding ground truth 3D meshes as supervisory signals. This, however, requires either annotation efforts~\cite{xiang2016objectnet3d} or synthetic construction~\cite{chang2015shapenet}. To avoid 3D supervision, \cite{kato2018neural,liu2019soft,kato2019learning,chen2019learning} consider an analysis-by-synthesis approach with differentiable rendering, but they still require either multi-view images or known camera poses. To further reduce supervision, \cite{kanazawa2018learning} learns category-specific 3D template shapes from an annotated image collection, but annotated 2D keypoints are still necessary in order to infer camera pose correctly. \cite{henderson2020learning} also studies a similar category-specific 3D reconstruction from a single image. \cite{li2020self} estimates 3D mesh, texture and camera pose of both rigid and non-rigid objects from a single-view image using silhouette as supervision. Videos~\cite{agrawal2015learning,zhou2017unsupervised,novotny2017learning,wang2018learning,wen2021self} are also leveraged as a form of supervision for single-view 3D reconstruction. For human bodies and faces, \cite{kanazawa2018end,gerig2018morphable,wang2019adversarial,gecer2019ganfit,yi2019mmface,choutas2020monocular,feng2021learning,feng2021collaborative,wen2021self,blanz1999morphable} reconstruct 3D shapes from single-view images with a predefined shape model such as SMPL~\cite{loper2015smpl}, FLAME~\cite{li2017learning} or BFM~\cite{paysan20093d}. Among many works in single-view 3D reconstruction, we are particularly interested in a simple and generic unsupervised framework from \cite{wu2020unsupervised} that utilizes the symmetric object prior. This framework adopts the Shape-from-shading pipeline to extract intrinsic factors of images, including 3D shape, texture, viewpoint and illumination parameters (as shown in Fig.~\ref{overview}(a)). The encoders are trained to minimize the reconstruction error between the input image and the rendered image. It shows impressive results in reconstructing human faces, cat faces and synthetic cars. 

For the sake of simplicity, we build the SCR pipeline based on the framework of \cite{wu2020unsupervised} and focus on studying how the causal structure of latent factors affects the 3D reconstruction performance. We emphasize that our method is a parallel contribution to \cite{wu2020unsupervised} and is generally applicable to any 3D reconstruction framework without the need of significant modifications.

\section{Causal Ordering of Latent Factors Matters}
\vspace{-1.5mm}

The very first question we need to address is ``\emph{Does the causal ordering of latent factors matter for unsupervised 3D reconstruction?}''. Without an affirmative answer, it will be pointless to study how to learn a good causal ordering. 
\vspace{-3.4mm}

\subsection{A Motivating Example from Function Approximation}
\vspace{-0.5mm}

\setlength{\columnsep}{13pt}
\begin{wrapfigure}{r}{0.295\textwidth}
  \begin{center}
  \advance\leftskip+1mm
  \renewcommand{\captionlabelfont}{\footnotesize}
    \vspace{-0.48in}  
    \includegraphics[width=0.268\textwidth]{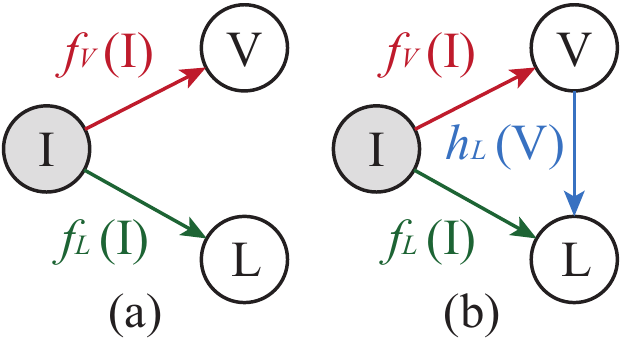}
    \vspace{-0.05in} 
    \caption{\footnotesize Two structures encode the lighting factor.}\label{example2}
    \vspace{-0.4in} 
  \end{center}
\end{wrapfigure}

We start with a motivating example to show the advantages of modeling the dependency between latent factors. We take a look at the example in Fig.~\ref{example2} where the lighting factor $\bm{L}$ can be represented using either $f_L(\bm{I})$ in Fig.~\ref{example2}(a) or $f_L(\bm{I})+h_L(\bm{V})$ in Fig.~\ref{example2}(b). There are a few perspectives to compare these two representations and see their difference (also see Appendix~\ref{app_formal_func}):

\vspace{-1.9mm}
\begin{itemize}
    \item We first assume the underlying data generating function for lighting is given by $\bm{L}:=f_L^*(\bm{I})+h_L^*(\bm{V})$ where $f_L^*$ and $h_L^*$ are two polynomial functions of order $p$. Because $\bm{V}:=f_V^*(\bm{I})$ where $f_V^*$ is also a polynomial function of order $p$, we can then write the lighting function as $\bm{L}:=f_L^*(\bm{I})+h_L^*\circ f_V^*(\bm{I})$ which is a polynomial function order $2p$.
    The lighting function can be learned with either $\bm{L}=f_L(\bm{I})$ in Fig.~\ref{example2}(a) or $\bm{L}=f_L(\bm{I})+h_L\circ f_V(\bm{I})$ in Fig.~\ref{example2}(b). The previous requires the encoder $f_L(\bm{I})$ to learn a polynomial of order $2p$, while the latter requires learning that of only order $p$.
    \item 
    From the perspective of function approximation, it is obvious that $f_L(\bm{I})+h_L\circ f_V(\bm{I})$ is always more expressive than $f_L(\bm{I})$ given that $f_L,h_L,f_V$ are of the same representation capacity. Therefore, the structure shown in Fig.~\ref{example2}(b) is able to capture more complex and nonlinear lighting function.
    \item Making the lighting $\bm{L}$ partially dependent on the viewpoint $\bm{V}$ gives the lighting function an inherent structural prior, which may implicitly regularizes the function class and constrain its inductive bias.
\end{itemize}

\vspace{-6.5mm}

\subsection{Expressiveness of Representing Conditional Distributions}
\vspace{-0.5mm}

The flexibility of SCR can also be interpreted from a distribution perspective. Most existing 3D reconstruction pipelines can be viewed as independent SCR whose conditional distribution $P(\bm{V},\bm{D},\bm{L},\bm{A}|\bm{I})$ can be factorized into
\begin{equation}
\small
P(\bm{V},\bm{D},\bm{L},\bm{A}|\bm{I})=P(\bm{V}|\bm{I})\cdot P(\bm{D}|\bm{I})\cdot P(\bm{L}|\bm{I})\cdot P(\bm{A}|\bm{I})
\end{equation}
which renders the conditional independence among $\bm{V},\bm{D},\bm{L},\bm{A}$. This is in fact a strong assumption that largely constrains the potential family of distributions that can factorize over this model, making this model less expressive in representing conditional distributions. In contrast, dense SCR does not assume any conditional independence because it yields the following factorization (this is just one of the potential orderings and we randomly choose one for demonstration):
\begin{equation}
\small
P(\bm{V},\bm{D},\bm{L},\bm{A}|\bm{I})=P(\bm{V}|\bm{I})\cdot P(\bm{D}|\bm{I},\bm{V})\cdot P(\bm{L}|\bm{I},\bm{V},\bm{D})\cdot P(\bm{A}|\bm{I},\bm{V},\bm{D},\bm{L})
\end{equation}
which imposes no constraints to the factorized conditional distribution and is more expressive. Therefore, any dense ordering has this nice property of assuming no conditional independence among latent factors. However, there exists a trade-off between expressiveness and learnability. A more expressive model usually requires more data to train and is relatively sample-inefficient. Generic SCR is proposed in search of a sweet spot between expressiveness and learnability by incorporating partial conditional independence. Taking Fig.~\ref{overview}(c) as an example, we can observe that this model assumes $P(\bm{D}\perp\bm{L}|\bm{I})$ and $P(\bm{D}\perp\bm{A}|\bm{I})$. Going beyond generic SCR, dynamic SCR aims to tackle with the scenario where the conditional distribution $P(\bm{V},\bm{D},\bm{L},\bm{A}|\bm{I})$ is dynamically changing rather than being static for all the images. This can greatly enhance the modeling flexibility.

\vspace{-3.7mm}
\subsection{Modeling Causality in Rendering-based Decoding}\label{causal}
\vspace{-1.3mm}
\setlength{\columnsep}{12pt}
\begin{wrapfigure}{r}{0.41\textwidth}
  \begin{center}
  \advance\leftskip+1mm
  \renewcommand{\captionlabelfont}{\footnotesize}
    \vspace{-0.42in}  
    \includegraphics[width=0.397\textwidth]{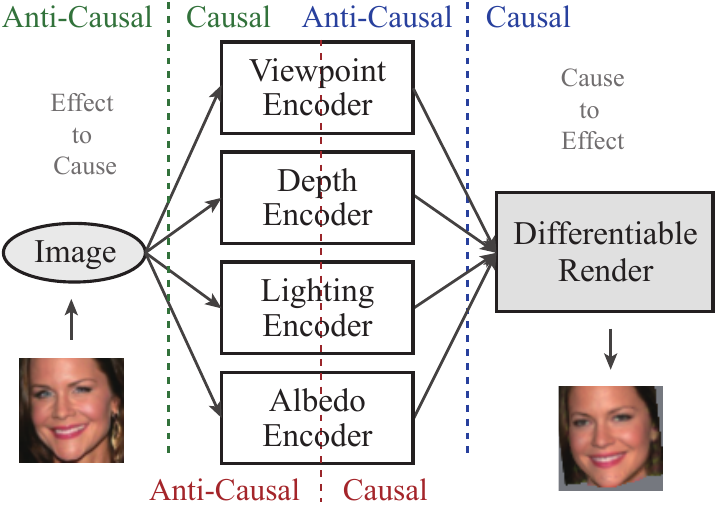}
    \vspace{-0.2in} 
    \caption{\footnotesize Three possible partitions of anti-causal and causal mappings.}\label{partition}
    \vspace{-0.3in} 
  \end{center}
\end{wrapfigure}

The previous subsection shows that there is no difference for different dense orderings in representing $P(\bm{V},\bm{D},\bm{L},\bm{A}|\bm{I})$. This conclusion is drawn from the perspective of modeling correlation. However, one of the most significant properties of topological ordering is its ability to model acyclic causality. In terms of causal relationships, different orderings (including both dense and generic ones) make a difference. The standard 3D reconstruction pipeline is naturally an autoencoder architecture, where the encoder and decoder can be interpreted as anti-causal and causal mappings, respectively~\cite{scholkopf2021toward,SchJanPetSgoetal12,weichwald2014causal,besserve2018intrinsic,kilbertus2018generalization,Leeb-SAE}. Here, the causal part is a generative mapping, and the anti-causal part is in the opposite direction, inferring causes from effects. However, how to determine which part of the pipeline should be viewed as anti-causal or causal remains unclear. Here we discuss three possible partitions of anti-causal and causal mappings, as shown in Fig.~\ref{partition}. The partition denoted by green dashed line uses an identity mapping as the anti-causal direction and the rest of the pipeline performs causal reconstruction. This partition does not {\parfillskip0pt\par}

\setlength{\columnsep}{13pt}
\begin{wrapfigure}{r}{0.35\textwidth}
  \begin{center}
  \advance\leftskip+1mm
  \renewcommand{\captionlabelfont}{\footnotesize}
    \vspace{-0.39in}  
    \includegraphics[width=0.34\textwidth]{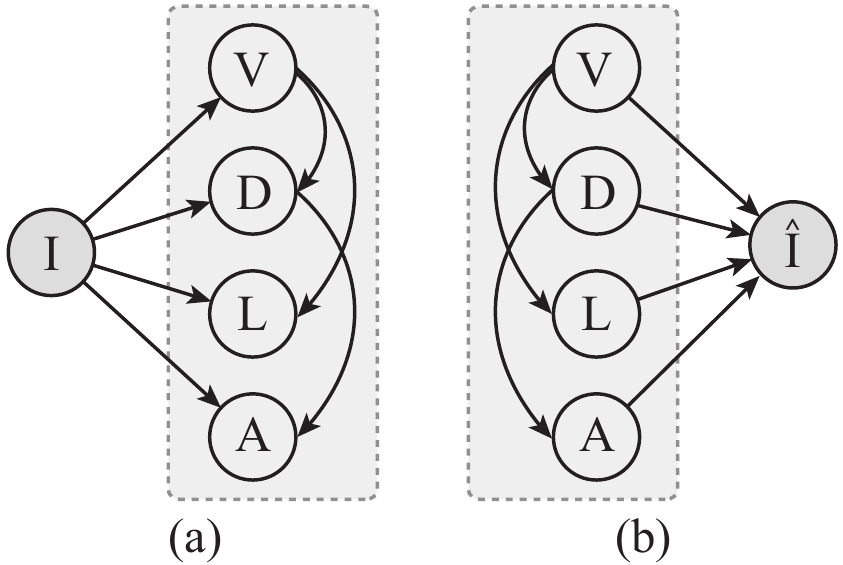}
    \vspace{-0.24in} 
    \caption{\footnotesize Latent structure modeling from (a) anti-causal direction and (b) causal direction. Gray regions denote where the causal ordering is learned.
    }\label{causal_anticausal}
    \vspace{-0.41in} 
  \end{center}
\end{wrapfigure}
\noindent explicitly model the causes and may not be useful. For the partition labeled by the blue dashed line, all the encoders are viewed as anti-causal, so the latent factor ordering is also part of anti-causal learning and does not necessarily benefit from the underlying causal ordering (\ie, causal DAG~\cite{vowels2021d}, cf.\ \cite{Leeb-SAE}). The partition denoted by the red dashed views part of the encoder as anti-causal learning and the rest of the encoder along with the renderer as causal learning. This partition is particularly interesting because it puts the latent factor ordering to the causal direction and effectively connects latent factor ordering to the underlying causal ordering. Our SCR framework (in Section~\ref{sectlfo}) is designed based on such insight. When the underlying causal ordering is available, using it as the default ordering could be beneficial. Although the causal ordering could improve strong generalization~\cite{kilbertus2018generalization}, learning the causal ordering without additional knowledge (\eg, interventions or manipulations such as randomized experiment) is difficult and out of our scope. \emph{We hypothesize that the underlying causal ordering leads to fast, generalizable and disentangled 3D reconstruction, and learning causal ordering based on these criteria may help us identify crucial causal relations.} As an encouraging signal, one of the best-performing dense ordering (DAVL) well matches the conventional rendering procedures in OpenGL, which is likely to be similar to the underlying causal ordering.

In the previous examples of Fig.~\ref{intro-example} and Fig.~\ref{example2}, we justify the necessity of the topological ordering from the factor estimation (\ie, anti-causal) perspective. As discussed above, we can alternatively incorporate the causal ordering to the causal mapping and model the causality among latent factors in the decoding (\ie, generative) process, which well matches the design of structural causal models. This is also conceptually similar to \cite{yang2021causalvae,shen2020disentangled} except that SCR augments the decoder with a physics-based renderer. Fig.~\ref{causal_anticausal} shows two interpretations of latent factor ordering from the causal and anti-causal directions. While the causal mapping encourages SCR to approximate the underlying causal ordering, the anti-causal mapping does not necessarily do so. The final learned causal ordering may be the result of a trade-off between causal and anti-causal mapping.

\vspace{-3.6mm}
\subsection{Empirical Evidence on 3D Reconstruction}
\vspace{-.9mm}

\begin{figure}[t]
  \renewcommand{\captionlabelfont}{\footnotesize}
  \setlength{\abovecaptionskip}{4pt}
  \setlength{\belowcaptionskip}{-10pt}
  \centering
  \vspace{-2.3mm}
  \includegraphics[width=4.7in]{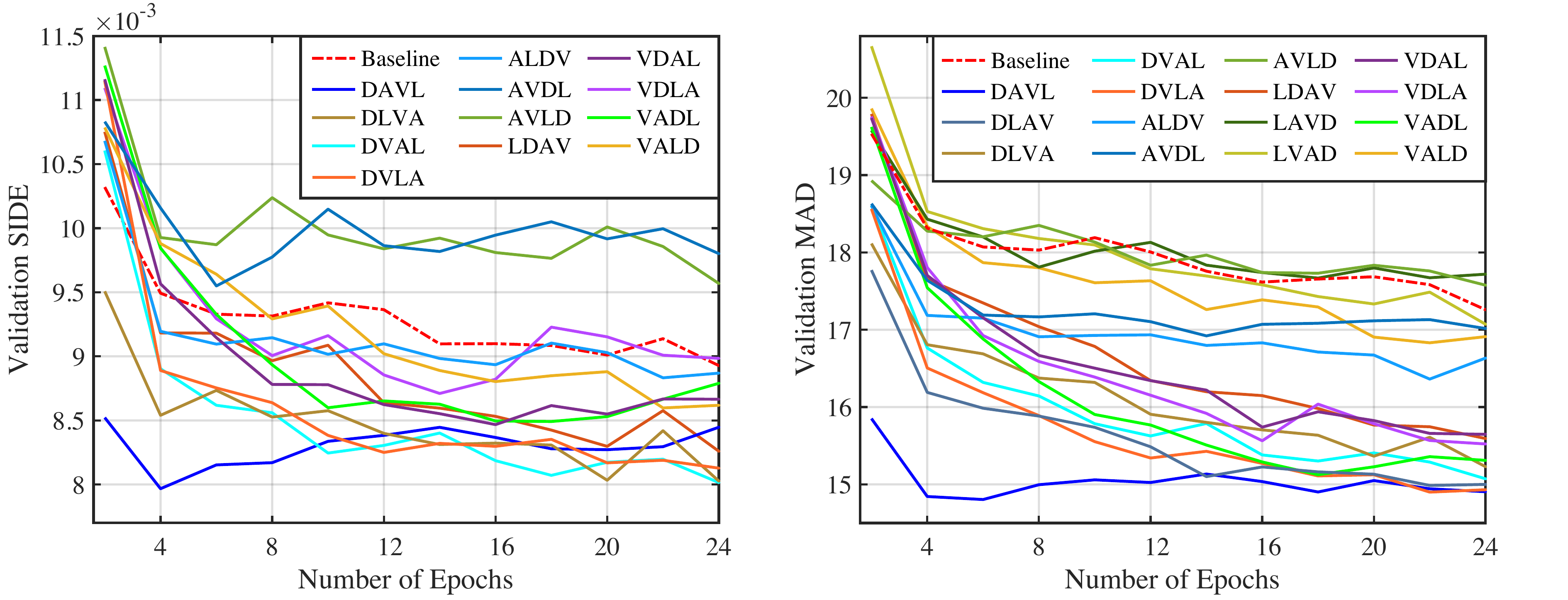}
  \caption{\footnotesize The scale-invariant depth error (left) and mean angle deviation (right) on the BFM dataset~\cite{paysan20093d} for different dense causal orderings. For visualization clarity, we plot the SIDE of the best three orderings, the worst three orderings and random six orderings. For MAD, we plot the same selection of orderings along with the best three and worst three orderings. We denote depth, albedo, lighting and viewpoint as D, A, L and V, respectively. For the full results, please refer to Appendix~\ref{app_dense_lfo}.}\label{empirical_evidence}
\end{figure}

Most importantly, we demonstrate the empirical performance of different dense causal orderings for unsupervised 3D reconstruction. The details of our pipeline and the experimental settings are given in Section~\ref{sectlfo} and Appendix~\ref{exp_detail}, respectively. Here we focus on comparing different dense orderings. As can be observed from Fig.~\ref{empirical_evidence}, different settings for dense SCR yield significantly different empirical behaviors, validating our claim that topological causal ordering of latent factors matters in unsupervised 3D reconstruction. Moreover, we discover that most of the dense orderings perform consistently for both SIDE and MAD metrics. For example, depth-albedo-viewpoint-lighting, depth-viewpoint-albedo-lighting and depth-viewpoint-lighting-albedo perform consistently better than the other dense orderings and the baseline (\ie, independent SCR). This again matches our intuition in Section~\ref{causal} that different dense ordering indicates different causality and leads to different disentanglement/reconstruction performance despite being equivalent in representing the conditional distribution $P(\bm{V},\bm{D},\bm{L},\bm{A}|\bm{I})$.

Interestingly, the well-performing dense orderings also seem to match our knowledge about the underlying causal ordering. For example, we also tend to put viewpoint in front of lighting, because the viewpoint will cause the change of lighting effects on the object. Almost all the well-performing dense orderings have this pattern, suggesting that the well-performing orderings tend to match the intrinsic causality that is typically hard to obtain in practice.

\vspace{-3.4mm}
\section{Learning Causal Ordering for 3D Reconstruction}
\vspace{-1.75mm}
We introduce a generic framework to learn causal ordering. Our proposed pipeline and algorithms to learn different variants of SCR are by no means optimal ones and it remains an open problem to learn a good causal ordering. We instead aim to show that a suitable causal ordering is beneficial to 3D reconstruction.

\vspace{-3.7mm}
\subsection{General SCR Framework}\label{sectlfo}
\vspace{-1.3mm}

Our unsupervised 3D reconstruction pipeline is inspired by \cite{wu2020unsupervised} but with some novel modifications to better accommodate the learning of causal ordering. Our goal is to study how causal ordering affects the disentanglement and generalizability in 3D reconstruction rather than achieving state-of-the-art performance.

\vspace{1mm}

\noindent\textbf{Decoding from a common embedding space}. A differentiable renderer typically takes in latent factors of different dimensions, making it less convenient to incorporate causal factor ordering. In order to easily combine multiple latent factors, we propose a learnable decoding method that includes additional neural networks ($f_V^2,f_D^2,f_L^2,f_A^2$ shown in Fig.~\ref{pipeline}) to the differentiable renderer. These neural networks transform the latent factors from a common $d$-dimensional embedding space ($\bm{u}_V,\bm{u}_D,\bm{u}_L,\bm{u}_A$) to their individual dimensions ($\bm{V},\bm{D},\bm{L},\bm{A}$) such that the differentiable renderer can directly use them as inputs. 

\vspace{1mm}

\noindent\textbf{Implementing SCR in a common embedding space}. Since all the latent factors can be represented in a common embedding space of the same dimension, we now introduce how to implement SCR in this pipeline. We start by listing a few key desiderata: (1) all variants of SCR should have (roughly) the same number of trainable parameters as independent SCR (baseline) such that the comparison is meaningful; (2) learning SCR should be efficient, differentiable and end-to-end; (3) different structures among latent factors can be explored in a unified framework by imposing different constraints on the adjacency matrix.

\begin{figure}[t]
  \renewcommand{\captionlabelfont}{\footnotesize}
  \setlength{\abovecaptionskip}{5pt}
  \setlength{\belowcaptionskip}{-10pt}
  \centering
  \vspace{-2mm}
  \includegraphics[width=4.7in]{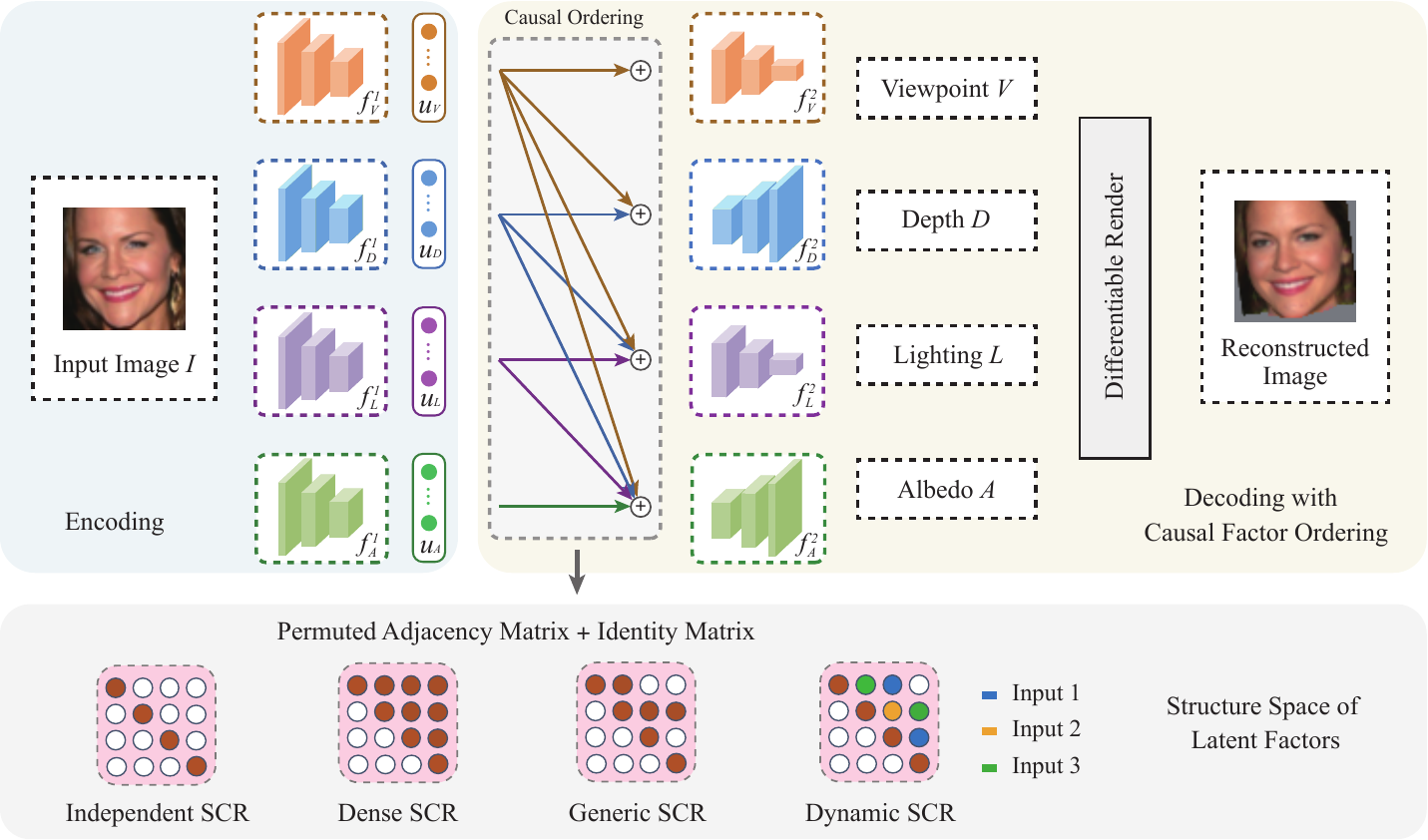}
  \caption{\footnotesize Our unsupervised 3D reconstruction pipeline to explore causal ordering. The causal edges in the figure are for illustration. Actual edges are learned in practice.}\label{pipeline}
\end{figure}

We first interpret conditional probability in terms of neural networks. For example, $P(\bm{V}|\bm{I},\bm{D})$ can be implemented as a single neural network $\bm{V}=f_V(\bm{I},\bm{D})$ that takes both image $\bm{I}$ and depth $\bm{D}$ as input. Instead of parameterizing the encoder $f_V$ with one neural network, we separate $f_V$ into two neural networks $f_V^1,f_V^2$ -- the first one $f_V^1$ aims to map different factors into a common embedding space of the same dimension, and the second one $f_V^2$ transforms the embedding to the final factor that can be used directly for the differentiable renderer. Taking $P(\bm{V}|\bm{I},\bm{D})$ as an example, we model it using $\bm{V}=f_V^2(f_V^1(\bm{I}),\bm{u}_D)$. We define the SCR adjacency matrix that characterizes the dependency structure among latent factors as $\bm{M}=[\bm{M}_V,\bm{M}_D,\bm{M}_L,\bm{M}_A]\in\mathbb{R}^{4\times 4}$ where $\bm{M}_V=[M_{VV},M_{VD},M_{VL},M_{VA}]^\top\in\mathbb{R}^{4\times 1}$ and $M_{VD}$ denotes the weight of the directed edge from $\bm{V}$ to $\bm{D}$ (the weight can be constrained to be either binary or continuous). Because causal ordering is equivalent to DAG, $\bm{M}$ can be permuted into a strictly upper triangular matrix. Generally, latent factors are modeled by

\vspace{-2.5mm}
\begin{equation}
\small
\begin{aligned}
    \bm{V}=f_V^2\left(f_V^1\left( \bm{I} \right),\bm{M}_V^\top\bm{u}\right)~~~~&~~~~\bm{D}=f_D^2\left(f_D^1\left( \bm{I} \right),\bm{M}_D^\top\bm{u}\right)\\
    \bm{L}=f_L^2\left(f_L^1\left( \bm{I} \right),\bm{M}_L^\top\bm{u}\right)~~~~&~~~~\bm{A}=f_A^2\left(f_A^1\left( \bm{I} \right),\bm{M}_A^\top\bm{u}\right)
\end{aligned}
\end{equation}
\vspace{-2mm}

\noindent where $\bm{u}=[\bm{u}_V;\bm{u}_D;\bm{u}_L;\bm{u}_A]\in\mathbb{R}^{4\times d}$. The input to $f_V^2,f_D^2,f_L^2,f_A^2$ can either be added element-wisely or concatenated, and we use element-wise addition in order not to introduce additional parameters. $\bm{M}$ exactly implements causal ordering as an equivalent form of causal DAG. More generally, $\bm{M}$ characterizes 
the latent space structure and can also be constrained to be some other family of structures.

\vspace{1mm}

\noindent\textbf{Interpreting SCR as a part of causal mapping}. After modeling the latent factors with two separate neural networks, we can view $f_V^1,f_D^1,f_L^1,f_A^1$ as the encoding process (\ie, the light blue region in Fig.~\ref{pipeline}). Different from \cite{wu2020unsupervised}, we view the causal ordering, $f_V^2,f_D^2,f_L^2,f_A^2$ and differentiable renderer as the decoding process (\ie, the light yellow region in Fig.~\ref{pipeline}). This can be understood as an augmented trainable 
physics-based renderer which performs rendering with additional neural networks and a causal ordering. More importantly, incorporating causal ordering to the decoding process makes it a part of causal mapping, which may produce more interpretable ordering due to its intrinsic connection to the underlying causality. Therefore, our novel pipeline design makes it possible to benefit from (or even estimate) the underlying causal ordering.

\vspace{1mm}
\noindent\textbf{Loss functions}. To avoid introducing additional priors to SCR and better study the effect of causal ordering, we stick to the same loss functions as \cite{wu2020unsupervised}. The loss function is defined as  $\mathcal{L}=\mathcal{L}_{\text{rec}}(\hat{\bm{I}},\bm{I})+\lambda_f\mathcal{L}_{\text{rec}}(\hat{\bm{I}}',\bm{I})+\lambda_p\mathcal{L}_{\text{p}}(\hat{\bm{I}},\bm{I})$ where $\mathcal{L}_{\text{rec}}$ is the reconstruction loss and $\mathcal{L}_{\text{p}}$ is the perceptual loss. $\lambda_f,\lambda_p$ are hyperparameters $\hat{\bm{I}}$ is the reconstructed image with original depth and albedo. $\hat{\bm{I}}'$ is the reconstructed image with flipped depth and albedo. Similar to \cite{wu2020unsupervised}, we also use the confidence map to compensate asymmetry. Appendix~\ref{exp_detail} provides the detailed formulation.

\vspace{1mm}
\noindent\textbf{Learning causal ordering}. We formulate the SCR learning as a bi-level optimization where the inner optimization is to train the 3D reconstruction networks with $\mathcal{L}$ and the outer optimization learns a suitable adjacency matrix $\bm{M}$:

\vspace{-4mm}
\begin{equation}\label{opt_obj}
\small
\min_{\bm{M}\in \mathcal{M}_{\text{DAG}}} \mathcal{L}_{\text{val}}(\bm{W}^*(\bm{M}),\bm{M})~~~~~~\text{s.t.}~\bm{W}^*(\bm{M})=\arg\min_{\bm{W}} \mathcal{L}_{\text{train}}(\bm{W},\bm{M})
\end{equation}
\vspace{-4mm}

\noindent where $\bm{W}$ denotes all the trainable parameters of neural networks in the 3D reconstruction pipeline, including $f_V^1,f_V^2,f_D^1,f_D^2,f_L^1,f_L^2,f_A^1,f_A^2$. $\mathcal{L}_{\text{train}}$ is the loss $\mathcal{L}$ computed on the training set, and $\mathcal{L}_{\text{val}}$ is the loss $\mathcal{L}$ computed on the validation set. Optionally, $\mathcal{L}_{\text{val}}$ may also include other supervised losses (\eg, ground truth depth) if available. This is in general a difficult problem, and in order to solve it effectively, we propose different algorithms based on the properties of the feasible set $\mathcal{M}_{\text{DAG}}$. After $\bm{M}$ is learned, we will fix $\bm{M}$ and retrain the network.

\vspace{-3.9mm}
\subsection{Learning Dense SCR via Bayesian Optimization}\label{dense_bo}
\vspace{-1mm}

The adjacency matrix $\bm{M}$ for dense SCR is an all-one strictly upper triangular matrix after proper permutation. Therefore, we are essentially learning the ordering permutation which is a discrete and non-differentiable structure. We resort to Bayesian optimization (BO)~\cite{frazier2018tutorial} that is designed for gradient-free and ``expensive to evaludate'' optimization.
Specifically, BO first places a Gaussian process prior on $\mathcal{L}_{\text{val}}(\bm{W}^*(\bm{M}),\bm{M})$ in Eq.~\eqref{opt_obj} and collect all the evaluated points on $\bm{M}$. Then BO updates posterior probability distribution on $\mathcal{L}_{\text{val}}$ using all available data and evaluates $\mathcal{L}_{\text{val}}$ on the maximizer point of the acquisition function which is computed with the current posterior distribution. Note that, evaluation on $\mathcal{L}_{\text{val}}$ requires computing $\bm{W}^*(\bm{M})$. Finally, BO outputs the latest evaluated $\bm{M}$. We use the position permutation kernel $K(\pi_1,\pi_2|\lambda)=\exp(-\lambda\cdot\sum_i |\pi^{-1}_1(i)-\pi^{-1}_2(i)|)$ where $\pi$ is a permutation mapping that maps the original index to the permuted index. We use the expected improvement as the acquisition function.  We note a special advantage of BO over gradient-dependent methods: the validation metric can be obtained from user study, which is often more reliable and flexible.

\vspace{-3.9mm}
\subsection{Learning Generic SCR via Optimization Unrolling}
\vspace{-1mm}

To solve the bi-level optimization in Eq.~\eqref{opt_obj}, we can unroll the inner optimization with a few gradient updates and replace $\bm{W}^*(\bm{M})$ with $\bm{W}-\eta\nabla_{\bm{W}}\mathcal{L}_{\text{train}}(\bm{W},\bm{M})$. Then the optimization becomes $\min_{\bm{M}\in \mathcal{M}_{\text{DAG}}}\mathcal{L}_{\text{val}}(\bm{W}-\eta\nabla_{\bm{W}}\mathcal{L}_{\text{train}}(\bm{W},\bm{M}),\bm{M})$. Here we unroll 1-step gradient update as an example, but we can also unroll multiple steps for better performance in practice. In order to constrain the adjacency matrix $\bm{M}$ to be a DAG, we can turn the feasible set $\bm{M}\in\mathcal{M}_{\text{DAG}}$ into a constraint~\cite{zheng2018dags}: $\mathcal{H}(\bm{M})=\text{tr}((\bm{I}_n+\frac{c}{n}\bm{M}\circ\bm{M})^n)-n=0$ where $\bm{I}_n$ is an identity matrix of size $n$, $c$ is some arbitrary positive number, $n$ is the number of latent factors (here $n=4$) and $\circ$ denotes the element-wise multiplication. Using Lagrangian multiplier method, we end up with the following optimization:

\vspace{-2.25mm}
\begin{equation}\label{unroll_obj}
\small
    \min_{\bm{M}}\mathcal{L}_{\text{val}}(\bm{W}-\eta\nabla_{\bm{W}}\mathcal{L}_{\text{train}}(\bm{W},\bm{M}),\bm{M})+\lambda_{\text{DAG}}\mathcal{H}(\bm{M})
\end{equation}
\vspace{-3.5mm}

\noindent where $\lambda_{\text{DAG}}$ is a hyperparameter. Alternatively, we may use the augmented Lagrangian method for stronger regularization \cite{zheng2018dags,yu2019dag}. Although Eq.~\eqref{unroll_obj} is easy to optimize, it is still difficult to guarantee the learned $\bm{M}$ to be a strict DAG and the search space may also be too large. To address this, we further propose a different approach to learn $\bm{M}$. The basic idea is to learn generic SCR based on the solution from dense SCR. We simply need to relearn/remove some edges for the given dense ordering. The final optimization is given by

\vspace{-2.45mm}
\begin{equation}\label{unroll_obj_dense}
\small
\min_{\bm{M}}\mathcal{L}_{\text{val}}(\bm{W}-\eta\nabla_{\bm{W}}\mathcal{L}_{\text{train}}(\bm{W},\bm{M}\circ \bm{M}^*_{\text{dense}}),\bm{M}\circ \bm{M}^*_{\text{dense}})
\end{equation}
\vspace{-3.05mm}

\noindent where $\bm{M}^*_{\text{dense}}$ is obtained from BO for dense SCR. It is a binary matrix that can be permuted to be strictly upper triangular. If we also constrain $\bm{M}$ to be binary, we will use a preset threshold to binarize the obtained $\bm{M}$ before retraining.

\vspace{-3.9mm}
\subsection{Learning Dynamic SCR via Masked Self-Attention}
\vspace{-1mm}

In order to make the adjacency matrix $\bm{M}$ be adaptively dependent on the input, we need to turn $\bm{M}$ into the output of a function that takes the image $\bm{I}$ as input, \ie, $\bm{M}=\Phi(\bm{I})$. One sensible choice is to parameterize $\Phi(\cdot)$ with an additional neural network, but it will inevitably introduce significantly more parameters and increase the capacity of the framework, making it unfair to compare with the other variants. Therefore, we take a different route by utilizing self-attention to design $\Phi(\cdot)$. Specifically, we use $\bm{M}=\Phi(\bm{I})=q(\bm{u})\circ\bm{M}^*_{\text{dense}}$ where $\bm{u}$ is the matrix containing all the factor embeddings ($\bm{u}=[\bm{u}_V;\bm{u}_D;\bm{u}_L;\bm{u}_A]$), $q(\bm{u})$ can be either the Sigmoid activation $\sigma(\bm{u}\bm{u}^\top)$ or cosine cross-similarity matrix among $\bm{u}_V,\bm{u}_D,\bm{u}_L,\bm{u}_A$ (\ie, $q(\bm{u})_{i,j}=\frac{\langle\bm{u}_i,\bm{u}_j\rangle}{\|\bm{u}_i\|\|\bm{u}_j\|},~i,j\in\{V,D,L,A\}$), and $\bm{M}^*_{\text{dense}}$ is the solution obtained from BO for dense SCR. $\bm{M}^*_{\text{dense}}$ essentially serves as a mask for the self-attention such that the resulting causal ordering is guaranteed to be a DAG. Since there is no fixed $\bm{M}$, the entire pipeline is trained with the final objective function:  $\min_{\bm{W}}\mathcal{L}_{\text{train}}(\bm{W},q(\bm{u})\circ\bm{M}^*_{\text{dense}})$ in an end-to-end fashion. We note that the function $q(\bm{u})$ has no additional parameters and meanwhile makes the causal ordering (\ie, $\Phi(\bm{I})$) dynamically dependent on the input image $\bm{I}$.

\vspace{-3.8mm}
\subsection{Insights and Discussion}
\vspace{-1mm}

\noindent\textbf{Connection to neural architecture search}. We discover an intriguing connection between SCR and neural architecture search (NAS)~\cite{zoph2017neural,elsken2019neural,liu2019darts}. SCR can be viewed as a special case of NAS that operates on a semantically interpretable space (\ie, the dependency structure among latent factors), while standard NAS does not necessarily produce an interpretable architecture. SCR performs like a top-down NAS where a specific neural structure is derived from semantic dependency/causality and largely constrains the search space for neural networks without suffering from countless poor local minima like NAS does.

\vspace{1mm}
\noindent\textbf{Semantic decoupling in common embeddings}. In order to make SCR interpretable, we require the latent embeddings $\bm{u}_V,\bm{u}_D,\bm{u}_L,\bm{u}_A$ to be semantically decoupled. For example, $\bm{u}_V$ should contain sufficient information to decode $\bm{V}$. The semantic decoupling in the common embedding space can indeed be preserved. First, the DAG constraint can naturally encourage semantic decoupling. We take an arbitrary dense ordering (\eg, DAVL) as an example. $\bm{u}_D$ is the only input for $f^2_D$, so it contain sufficient information for $\bm{D}$. $\bm{u}_D,\bm{A}$ are the inputs for $f^2_A$, so the information of $\bm{A}$ will be largely encoded in $\bm{u}_A$ ($\bm{u}_D$ already encodes the information of $\bm{D}$). The same reasoning applies to $\bm{V}$ and $\bm{L}$. Note that, a generic DAG will have less decoupling than dense ordering due to less number of directed edges. Second, we enforce the encoders $f_V^2,f_D^2,f_L^2,f_A^2$ to be relatively simple functions (\eg, shallow neural networks), such that they are unable to encode too much additional information and mostly serve as dimensionality transformation. They could also be constrained to be invertible. Both mechanisms ensure the semantic decoupling in the common embedding space.

%\noindent\textbf{Alternative pipelines to incorporate LFO}.
\vspace{-3.5mm}
\section{Experiments and Results}
\vspace{-1.8mm}

\noindent\textbf{Datasets}. We evaluate our method on two human face datasets (CelebA~\cite{liu2015deep} and BFM~\cite{paysan20093d}), one cat face dataset that combines \cite{zhang2008cat} and \cite{parkhi2012cats} (cropped by \cite{wu2020unsupervised}) and one car dataset~\cite{wu2020unsupervised} rendered from ShapeNet~\cite{chang2015shapenet} with random viewpoints and illumination. These images are split 8:1:1 into training, validation and testing.

\vspace{1mm}

\noindent\textbf{Metrics}. For fairness, we use the same metrics as \cite{wu2020unsupervised}. The first one is Scale Invariant Depth Error (SIDE)~\cite{eigen2014depth} which computes the standard deviation of the difference between the estimated depth map at the input view and the ground truth depth map at the log scale. We note that this metric may not reflect the true reconstruction quality. As long as this metric is reasonably low, it may no longer be a stronger indicator for reconstruction quality, which is also verified by \cite{ho2021toward}. To make a comprehensive evaluation, we also use another metric: the mean angle deviation (MAD)~\cite{wu2020unsupervised} between normals computed from ground truth depth from the predicted depth. It measures how well the surface is reconstructed.

\vspace{1mm}
\noindent\textbf{Implementation}. For the network architecture, we follow \cite{wu2020unsupervised} and only make essential changes to its setup such that the comparison is meaningful. For the detailed implementation and experimental settings, refer to Appendix~\ref{exp_detail}.

\vspace{-3.7mm}
\subsection{Quantitative Results}
\vspace{-1mm}

\noindent\textbf{Geometry reconstruction}. We train and test all the methods on the BFM dataset to evaluate the depth reconstruction quality. The results are given in Table~\ref{bfm_results}. We compare different variants of SCR with our own baseline (\ie, independent SCR), two state-of-the-art methods~\cite{wu2020unsupervised,ho2021toward}, supervised learning upper bound, constant null depth and average ground truth depth. We note that there is a performance difference between our re-run version and the original version of \cite{wu2020unsupervised}. This is because all our experiments are run under CUDA-10 
\setlength{\columnsep}{8pt}
\begin{wraptable}{r}[0cm]{0pt}
    \scriptsize
	\centering
	\setlength{\tabcolsep}{3pt}
	\renewcommand{\arraystretch}{1.2}
	\renewcommand{\captionlabelfont}{\footnotesize}
	%\vspace{0.7mm}
	\begin{tabular}{c|cc}
	\specialrule{0em}{0pt}{0pt}
		%\hline
		  Method & SIDE ($\times 10^{-2}$) $\downarrow$ & MAD (deg.) $\downarrow$\\
		  \shline
		   Supervised & \textbf{0.410} {\scriptsize$\pm$0.103} & \textbf{10.78} {\scriptsize$\pm$1.01} \\
           Constant Null Depth & 2.723 {\scriptsize$\pm$0.371} &  43.34 {\scriptsize$\pm$2.25} \\
           Average GT Depth &  1.990 {\scriptsize$\pm$0.556} &  23.26 {\scriptsize$\pm$2.85} \\\hline
           Wu et al.~\cite{wu2020unsupervised} (reported) & \textbf{0.793} {\scriptsize$\pm$0.140} & 16.51 {\scriptsize$\pm$1.56} \\
           Ho et al.~\cite{ho2021toward} (reported) & 0.834 {\scriptsize$\pm$0.169} & \textbf{15.49} {\scriptsize$\pm$1.50} \\
           Wu et al.~\cite{wu2020unsupervised} (our run) & 0.901 {\scriptsize$\pm$0.190} & 17.53 {\scriptsize$\pm$1.84} \\\hline\rowcolor{Gray}
           Independent SCR & 0.895 {\scriptsize$\pm$0.183} & 17.36 {\scriptsize$\pm$1.78} \\
            \rowcolor{Gray}
           Dense SCR (random) & 1.000 {\scriptsize$\pm$0.275} & 17.66 {\scriptsize$\pm$2.09}  \\\rowcolor{Gray}
           Dense SCR (BO) & \textbf{0.830} {\scriptsize$\pm$0.205} & \textbf{14.88} {\scriptsize$\pm$1.94}  \\\hline\rowcolor{Gray}
		   Generic SCR (Eq.~\ref{unroll_obj}) & 0.859 {\scriptsize$\pm$0.215} & 15.17 {\scriptsize$\pm$1.92}  \\\rowcolor{Gray}
		   Generic SCR (Eq.~\ref{unroll_obj_dense}) & \textbf{0.820} {\scriptsize$\pm$0.190} & \textbf{14.79} {\scriptsize$\pm$1.96}  \\\hline\rowcolor{Gray}
		   Dynamic SCR (Sigmoid) & 0.827 {\scriptsize$\pm$0.220} & 14.86 {\scriptsize$\pm$2.02}  \\\rowcolor{Gray}
		   Dynamic SCR (Cosine) & \textbf{0.815} {\scriptsize$\pm$0.232} & \textbf{14.80} {\scriptsize$\pm$1.95}  \\
		  \specialrule{0em}{-5pt}{0pt}
	\end{tabular}
	\caption{\footnotesize Depth reconstruction results on BFM.} \label{bfm_results}
\vspace{-1mm}
\end{wraptable}

\vspace{-4.2mm}
\noindent while the original version of \cite{wu2020unsupervised} is trained on CUDA-9. We also re-train our models on CUDA-9 and observe a similar performance boost (see Appendix~\ref{exp_cuda9}). We suspect this is because of the rendering precision on different CUDA versions. However, this will not affect the advantages of our method and our experiment settings are the same for all the other compared methods. More importantly, we build our SCR on a baseline that performs similarly to \cite{wu2020unsupervised} (independent SCR vs. our version of \cite{wu2020unsupervised}). SCR improves our baseline for more than $0.0065$ on SIDE and $2.5$ degree on MAD. Specifically, our dense SCR learns an ordering of depth-viewpoint-albedo-lighting. Generic SCR (Eq.~\ref{unroll_obj_dense}) and both dynamic SCR variants are built upon this ordering. We notice that if we use a random dense ordering, then the 3D reconstruction results are even worse than our baseline, which shows that dense SCR can indeed learn crucial structures. Such a significant performance gain shows that a suitable SCR can implicitly regularize the neural networks and thus benefit the 3D reconstruction.

\vspace{1.5mm}

\setlength{\columnsep}{6pt}
\begin{wraptable}{r}[0cm]{0pt}
    \scriptsize
	\centering
	\setlength{\tabcolsep}{3pt}
	\renewcommand{\arraystretch}{1.3}
	\renewcommand{\captionlabelfont}{\footnotesize}
	%\vspace{0.7mm}
	\begin{tabular}{c|ccc}
	\specialrule{0em}{0pt}{-22pt}
		%\hline
		  Method & DLVA & DAVL & DVAL\\
		  \shline
		   Dense SCR (fixed) & 15.02 {\scriptsize$\pm$2.00} & 15.14 {\scriptsize$\pm$1.91} & 14.88 {\scriptsize$\pm$1.94}  \\
		   G-SCR (Eq.~\ref{unroll_obj_dense}) & \textbf{14.96} {\scriptsize$\pm$1.90} & \textbf{14.85} {\scriptsize$\pm$2.13} & \textbf{14.79} {\scriptsize$\pm$1.96} \\
		   Dy-SCR (Sigmoid) & 15.01 {\scriptsize$\pm$1.99} & 15.03 {\scriptsize$\pm$2.12} & 14.86 {\scriptsize$\pm$2.02}  \\
		   Dy-SCR (Cosine) & 14.99 {\scriptsize$\pm$1.93} & 15.05 {\scriptsize$\pm$2.15} & 14.80 {\scriptsize$\pm$1.95}  \\
		  \specialrule{0em}{-6pt}{0pt}
	\end{tabular}
	\caption{\footnotesize MAD (degree) results on BFM.} \label{ablation}
\vspace{-4mm}
\end{wraptable}

\noindent\textbf{Effect of dense ordering}. We also perform ablation study to see how generic SCR and dynamic SCR perform if they are fed with different dense orderings. Table~\ref{ablation} compares three different dense orderings: depth-light-viewpoint-albedo (DLVA), depth-albedo-viewpoint-lighting (DAVL) and depth-viewpoint-albedo-lighting (DVAL). We show that G-SCR and D-SCR can consistently improve the 3D reconstruction results even if different dense orderings are given as the mask.

\vspace{1.5mm}

\setlength{\columnsep}{10pt}
\begin{wrapfigure}{r}{0.643\textwidth}
  \begin{center}
  \advance\leftskip+1mm
  \renewcommand{\captionlabelfont}{\footnotesize}
    \vspace{-0.48in}  
    \includegraphics[width=0.646\textwidth]{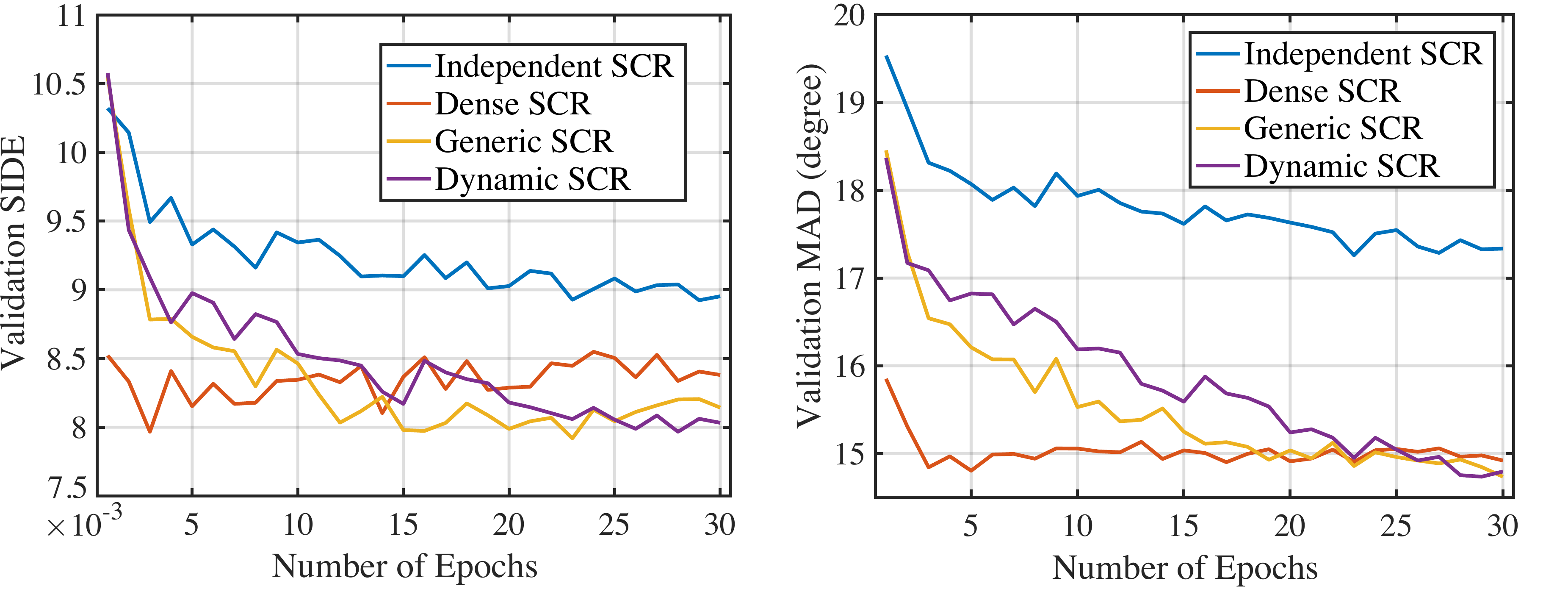}
    \vspace{-0.265in} 
    \caption{\footnotesize Convergence curves of validation SIDE \& MAD.}\label{convergence}
    \vspace{-0.35in} 
  \end{center}
\end{wrapfigure}

\noindent\textbf{Convergence}. Fig.~\ref{convergence} plots the convergence curves of both SIDE and MAD in Fig.~\ref{convergence}. We observe that dense SCR, generic SCR and dynamic SCR converge much faster than the baseline. Dense SCR achieves impressive performance at the very beginning of the training. When converged, dynamic SCR and generic SCR performs better than dense SCR due to its modeling flexibility.

\begin{figure}[t]
  \renewcommand{\captionlabelfont}{\footnotesize}
  \setlength{\abovecaptionskip}{3.5pt}
  \setlength{\belowcaptionskip}{-1.5pt}
  \centering
  \vspace{-2mm}
  \includegraphics[width=4.65in]{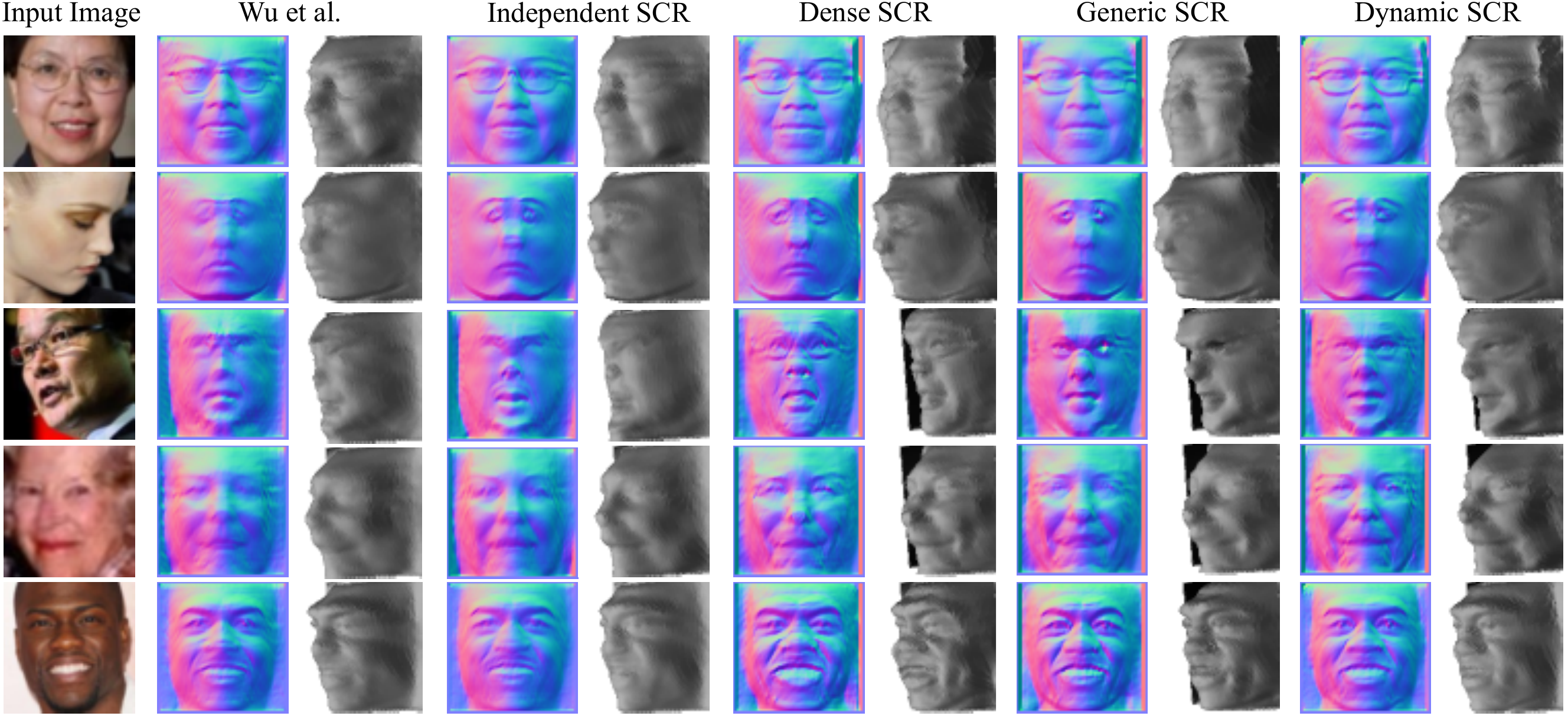}
  \caption{\footnotesize One textureless view and canonical normal map on CelebA. All the methods (including  Wu et al.~\cite{wu2020unsupervised}) are trained on CelebA under the same experimental settings.}\label{celeb_img}
\end{figure}

\begin{figure}[t]
  \renewcommand{\captionlabelfont}{\footnotesize}
  \setlength{\abovecaptionskip}{3.25pt}
  \setlength{\belowcaptionskip}{-15.25pt}
  \centering
  %\vspace{-0.5mm}
  \includegraphics[width=4.65in]{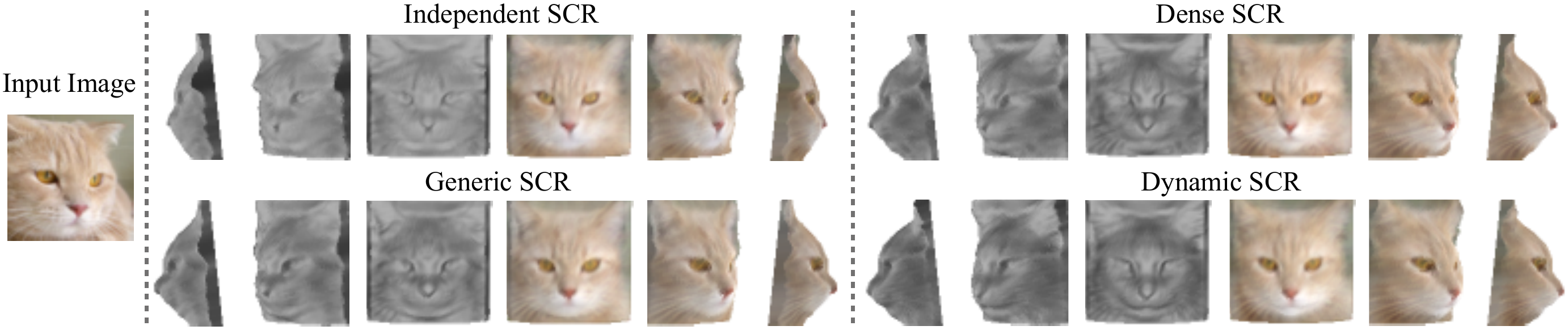}
  \caption{\footnotesize Textured and textureless shapes from multiple views on cat faces.}\label{cat_img}
\end{figure}

\vspace{-5.25mm}
\subsection{Qualitative Results}
\vspace{-1.35mm}

\noindent\textbf{CelebA}. We show the reconstruction results for a few challenging in-the-wild face images (\eg, extreme poses and expressions) in Fig.~\ref{celeb_img}. We train dense SCR with BO and the other SCR variants are trained based on the best learned dense ordering. Our SCR variants including dense SCR, generic SCR and dynamic SCR are able to reconstruct fine-grained geometric details and recover more realistic shapes than both \cite{wu2020unsupervised} and our independent baseline.
This well verifies the importance of implicit regularization from latent space structure.

\vspace{1mm}

\setlength{\columnsep}{9pt}
\begin{wrapfigure}{r}{0.49\textwidth}
  \begin{center}
  \advance\leftskip+1mm
  \renewcommand{\captionlabelfont}{\footnotesize}
    \vspace{-0.46in}  
    \includegraphics[width=0.488\textwidth]{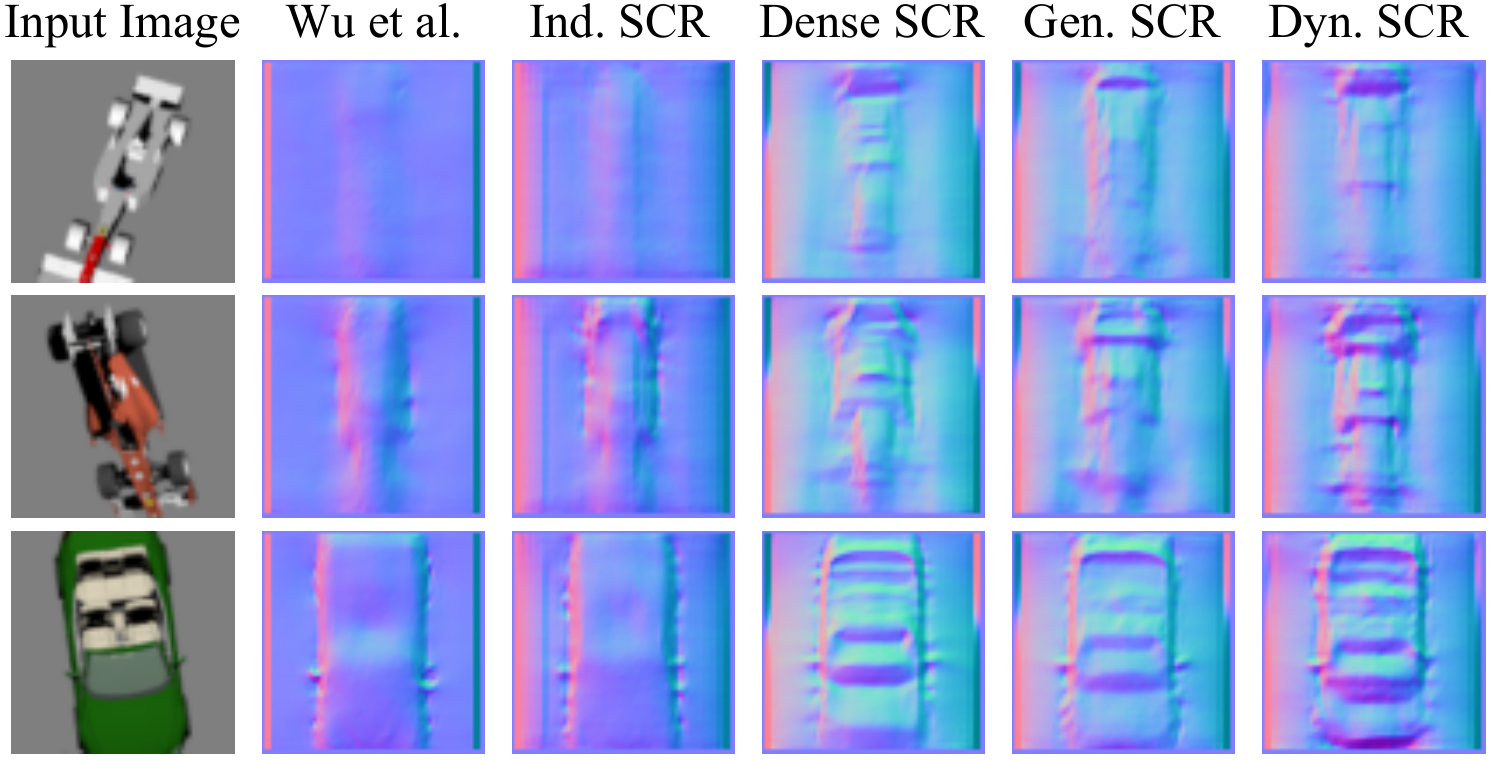}
    \vspace{-0.25in} 
    \caption{\footnotesize Canonical normal maps on cars.}\label{car_img}
    \vspace{-0.4in} 
  \end{center}
\end{wrapfigure}
\noindent\textbf{Cat faces}. We also train all the SCR variants on cat faces. Results in Fig.~\ref{cat_img} show that dynamic SCR yields the best 3D reconstruction quality, while dense SCR and generic SCR can also recover reasonably good geometric details.

\vspace{1mm}
\noindent\textbf{Cars}. We train all the methods on the synthetic car dataset under the same settings, and then evaluate these methods on car images with abundant geometric details. Fig.~\ref{car_img} shows that our SCR variants recovers very fine-grained geometric details and produce highly realistic normal maps which are significantly better than both \cite{wu2020unsupervised} and independent SCR.

\vspace{1.1mm}
{
\begin{spacing}{0.2}
\tiny\noindent\textbf{Acknowledgements}. This work was supported by the German Federal Ministry of Education and Research (BMBF): Tübingen AI Center, FKZ: 01IS18039A, 01IS18039B; and by the Machine Learning Cluster of Excellence, EXC number 2064/1 – Project number 390727645. AW is supported by a Turing AI Fellowship under grant EP/V025279/1, The Alan Turing Institute, and the Leverhulme Trust via CFI. LP is supported by the CIFAR CCAI Chairs Program.
\end{spacing}
}

%\section{Concluding Remarks}

\clearpage
\bibliographystyle{splncs04}
\bibliography{egbib}

\clearpage
\newpage
\appendix
\onecolumn

\begin{center}
    \textbf{\Large Appendix}
\end{center}

\section{Experimental Details}\label{exp_detail}

\subsection{Architectures and settings}

%For the attention based one, the key and query mapping are both matrices in $\mathbb{R}^{256 \times 256}$.

Our architectures generally follow \cite{wu2020unsupervised} with a few additional decoders. The specific architectures are given in Table~\ref{tab:arch_view_light}, Table~\ref{tab:arch_depth_albedo} and Table~\ref{tab:arch_conf}. We use the same training setting as mentioned in \cite{wu2020unsupervised} unless otherwise specified. We learn different variants of SCR on the BFM dataset and use it to the other datasets (CelebA, cat faces and cars) without retraining SCR on these datasets. This still obtains satisfactory performance even if we do not directly learn SCR on these datasets, which show \emph{strong transferrability of SCR}. This also partially justifies that SCR can indeed learn some underlying causal knowledge that can generalize across different domains. However, we want to emphasize that we can still apply SCR on individual datasets and obtain even larger improvement.

\begin{table}[h]
\footnotesize
\centering
\vspace{-2mm}
\begin{tabular}{lc}
\toprule
 Encoder ($f_V^1$ and $f_L^1$) & Output size\\ \midrule
 Conv(3, 32, 4, 2, 1) + ReLU & 32\\
 Conv(32, 64, 4, 2, 1) + ReLU & 16\\
 Conv(64, 128, 4, 2, 1) + ReLU & 8\\
 Conv(128, 256, 4, 2, 1) + ReLU & 4\\
 Conv(256, 256, 4, 1, 0) + ReLU & 1\\
 Conv(256, 256, 1, 1, 0) $\rightarrow output$ & 1\\ \midrule \midrule
 Decoder ($f_V^2$ and $f_L^2$) & Output size \\ \midrule
 MLP(256, 256) + ReLU & 1 \\
 MLP(256, $c_{out}$) + Tanh $\rightarrow output$ & 1\\
\bottomrule
\end{tabular}
\vspace{3mm}
\caption{\footnotesize Encoder ($f_V^1$ and $f_L^1$) and decoder ($f_V^2$ and $f_L^2$) network architecture for viewpoint and lighting. The output channel size $c_{out}$ is $6$ for viewpoint.}\label{tab:arch_view_light}
\vspace{-10mm}
\end{table}

\begin{table}[t]
\footnotesize
\centering
\begin{tabular}{lc}
\toprule
 Encoder ($f_D^1$ and $f_A^1$) & Output size \\ \midrule
 Conv(3, 64, 4, 2, 1) + GN(16) + LReLU(0.2) & 32\\
 Conv(64, 128, 4, 2, 1) + GN(32) + LReLU(0.2) & 16\\
 Conv(128, 256, 4, 2, 1) + GN(64) + LReLU(0.2) & 8\\
 Conv(256, 512, 4, 2, 1) + LReLU(0.2) & 4\\
 Conv(512, 256, 4, 1, 0) + ReLU & 1\\ \midrule \midrule
 Decoder ($f_D^2$ and $f_A^2$) & Output size \\ \midrule
 Deconv(256, 512, 4, 1, 0) + ReLU & 4\\
 Conv(512, 512, 3, 1, 1) + ReLU & 4\\
 Deconv(512, 256, 4, 2, 1) + GN(64) + ReLU & 8\\
 Conv(256, 256, 3, 1, 1) + GN(64) + ReLU & 8\\
 Deconv(256, 128, 4, 2, 1) + GN(32) + ReLU & 16\\
 Conv(128, 128, 3, 1, 1) + GN(32) + ReLU & 16\\
 Deconv(128, 64, 4, 2, 1) + GN(16) + ReLU & 32\\
 Conv(64, 64, 3, 1, 1) + GN(16) + ReLU & 32\\
 Upsample(2) & 64\\
 Conv(64, 64, 3, 1, 1) + GN(16) + ReLU & 64\\
 Conv(64, 64, 5, 1, 2) + GN(16) + ReLU & 64\\
 Conv(64, $c_{out}$, 5, 1, 2) + Tanh $\rightarrow output$ & 64\\
\bottomrule
\end{tabular}
\vspace{3mm}
\caption{\footnotesize Encoder (($f_D^1$ and $f_A^1$)) and decoder ($f_D^2$ and $f_A^2$) network architecture for depth and albedo. The output channel size $c_{out}$ is $1$ for depth and $3$ for albedo.}\label{tab:arch_depth_albedo}
\end{table}

\begin{table}[!t]
\footnotesize
\centering
\vspace{-2mm}
\begin{tabular}{lc}
\toprule
 Encoder & Output size \\ \midrule
 Conv(3, 64, 4, 2, 1) + GN(16) + LReLU(0.2) & 32\\
 Conv(64, 128, 4, 2, 1) + GN(32) + LReLU(0.2) & 16\\
 Conv(128, 256, 4, 2, 1) + GN(64) + LReLU(0.2) & 8\\
 Conv(256, 512, 4, 2, 1) + LReLU(0.2) & 4\\
 Conv(512, 128, 4, 1, 0) + ReLU & 1\\ \midrule \midrule
 Decoder & Output size \\ \midrule
 Deconv(128, 512, 4, 1, 0) + ReLU & 4\\
 Deconv(512, 256, 4, 2, 1) + GN(64) + ReLU & 8\\
 Deconv(256, 128, 4, 2, 1) + GN(32) + ReLU & 16\\
 \enskip \rotatebox[origin=c]{180}{$\Lsh$} Conv(128, 2, 3, 1, 1) + SoftPlus $\rightarrow output$ & 16\\
 Deconv(128, 64, 4, 2, 1) + GN(16) + ReLU & 32\\
 Deconv(64, 64, 4, 2, 1) + GN(16) + ReLU & 64\\
 Conv(64, 2, 5, 1, 2) + SoftPlus $\rightarrow output$ & 64\\
\bottomrule
\end{tabular}
\vspace{3mm}
\caption{\footnotesize Network architecture for confidence maps. The network outputs two pairs of confidence maps at different spatial resolutions for photometric and perceptual losses.}\label{tab:arch_conf}
\vspace{-6mm}
\end{table}

\subsection{Loss Formulation}

We adopt the loss functions from \cite{wu2020unsupervised}. For our paper to be self-contained, we provide the detailed loss formulation here. We note that all the variants including independent SCR, dense SCR, generic SCR and dynamic SCR use the same set of loss functions. Specially, we use
\begin{equation}
    \mathcal{L}=\mathcal{L}_{\text{rec}}(\hat{\bm{I}},\bm{I})+\lambda_f\mathcal{L}_{\text{rec}}(\hat{\bm{I}}',\bm{I})+\lambda_p\mathcal{L}_{\text{p}}(\hat{\bm{I}},\bm{I})
\end{equation}
where $\lambda_f$ and $\lambda_p$ are hyperparameters that weight the loss function. Specifically, we have that
\begin{equation}
    \mathcal{L}_{\text{rec}}(\hat{\bm{I}},\bm{I})=-\frac{1}{|\Omega|}\sum_{u,v\in\Omega}\ln\big(\frac{1}{\sqrt{2}\sigma_{u,v}}\big)\exp\big( -\frac{\sqrt{2}|\hat{\bm{I}}_{u,v}-\bm{I}_{u,v}|}{\sigma_{u,v}} \big)
\end{equation}
where $u,v$ denotes pixel locations and $\sigma$ is the confidence map produced by an additional neural network. This can also be viewed as aleatoric uncertainty. Similarly we also have that
\begin{equation}
    \mathcal{L}_{\text{rec}}(\hat{\bm{I}}',\bm{I})=-\frac{1}{|\Omega|}\sum_{u,v\in\Omega}\ln\big(\frac{1}{\sqrt{2}\sigma'_{u,v}}\big)\exp\big( -\frac{\sqrt{2}|\hat{\bm{I}}'_{u,v}-\bm{I}_{u,v}|}{\sigma'_{u,v}} \big)
\end{equation}
where $\sigma'$ is another confidence map that models the uncertainty of the symmetry, \ie, which part of the image might not be symmetric. Finally we have the perceptual loss for the input and reconstructed image:
\begin{equation}
    \mathcal{L}_{\text{p}}(\hat{\bm{I}},\bm{I})=-\frac{1}{|\Omega|}\sum_{u,v\in\Omega}\ln\big(\frac{1}{\sqrt{2}\sigma_{u,v}^2}\big)\exp\big( -\frac{(\phi(\hat{\bm{I}})_{u,v}-\phi(\bm{I})_{u,v})^2}{2\sigma_{u,v}^2} \big)
\end{equation}
where $\phi(\bm{I})$ denotes the feature map from one layer in VGG-16~\cite{simonyan2014very} (\texttt{relu3\_3}). In our experiments, we mostly follow the practice in \cite{wu2020unsupervised} by setting $\lambda_f=0.5$ and $\lambda_p=1$. Additionally, we also compute the perceptual loss between $\hat{\bm{I}}'$ and $\bm{I}$ with the confidence map $\sigma'$.

\subsection{Learning Dense SCR via Bayesian Optimization}

We use expected improvement as the acquisition function. For the hyperparameter optimizer (to fit the $\lambda$ parameter in the kernel function at each BO iteration), we use Adam optimizer and set the learning rate to $0.1$ and we train the hyperparameter till convergence. We run $5000$ training steps and collect the validation loss for each BO step. We run $10$ BO iterations and use the same validation metric as in Appendix~\ref{supp_unrolling}. After the BO is finished, we will use the learned dense ordering to retrain the model in order to obtain the final results of dense SCR.

\subsection{Learning Generic SCR via Optimization Unrolling}\label{supp_unrolling}

%We use a combined validation objective to train the DAG:
%$$
%    L_{\textnormal{val}} = L_{Image} + L_{Perp} + 10.0 * L_{NorErr} + \lambda_{\text{DAG}}\mathcal{H}(\bm{M})
%$$
%, in which $L_{Image}$ is the L1 image reconstruction loss, $L_{Perp}$ is the perceptual loss, $L_{NorErr}$ is the average normal angle error (same as the CVPR paper). In addition, we have the DAG regularization loss from \cite{zheng2018dags} to ensure that the graph is close to a DAG:
%$$
%    L_{\textnormal{DAG}}(A) = (Tr((I + \frac{1}{2}A^2)^m) - m)^2
%$$
%, where $A \in \mathbb{R}^{m \times m}$ is the adjacency matrix of the underlying graph.

We use Adam optimizer for training the DAG in the outer loops and set the initial learning rate of the DAG to $1e-3$ (same as that of all the other parameters). $\lambda_{\textnormal{DAG}}$ is initialized to 10. For each $2500$ iterations, we multiply $\lambda_{\textnormal{DAG}}$ by 10 and decrease the learning rate of the DAG optimizer by half. After training an experiment , we record the resulted DAG and discretize the adjancency matrix to a 0/1 matrix with a threshold of $0.01$. We retrain the model using this discretized and unweighted DAG. Optionally, we find that using the average normal angle error in the validation loss may greatly improve the performance. However, we stick to the case where no supervision is used to learn SCR for fair comparison.

\clearpage
\newpage
\section{Full Results of Figure~\ref{example2}}\label{app_dense_lfo}

\begin{figure}[h]
  \renewcommand{\captionlabelfont}{\footnotesize}
  \setlength{\abovecaptionskip}{8pt}
  \setlength{\belowcaptionskip}{0pt}
  \centering
  \vspace{-1mm}
  \includegraphics[width=4.7in]{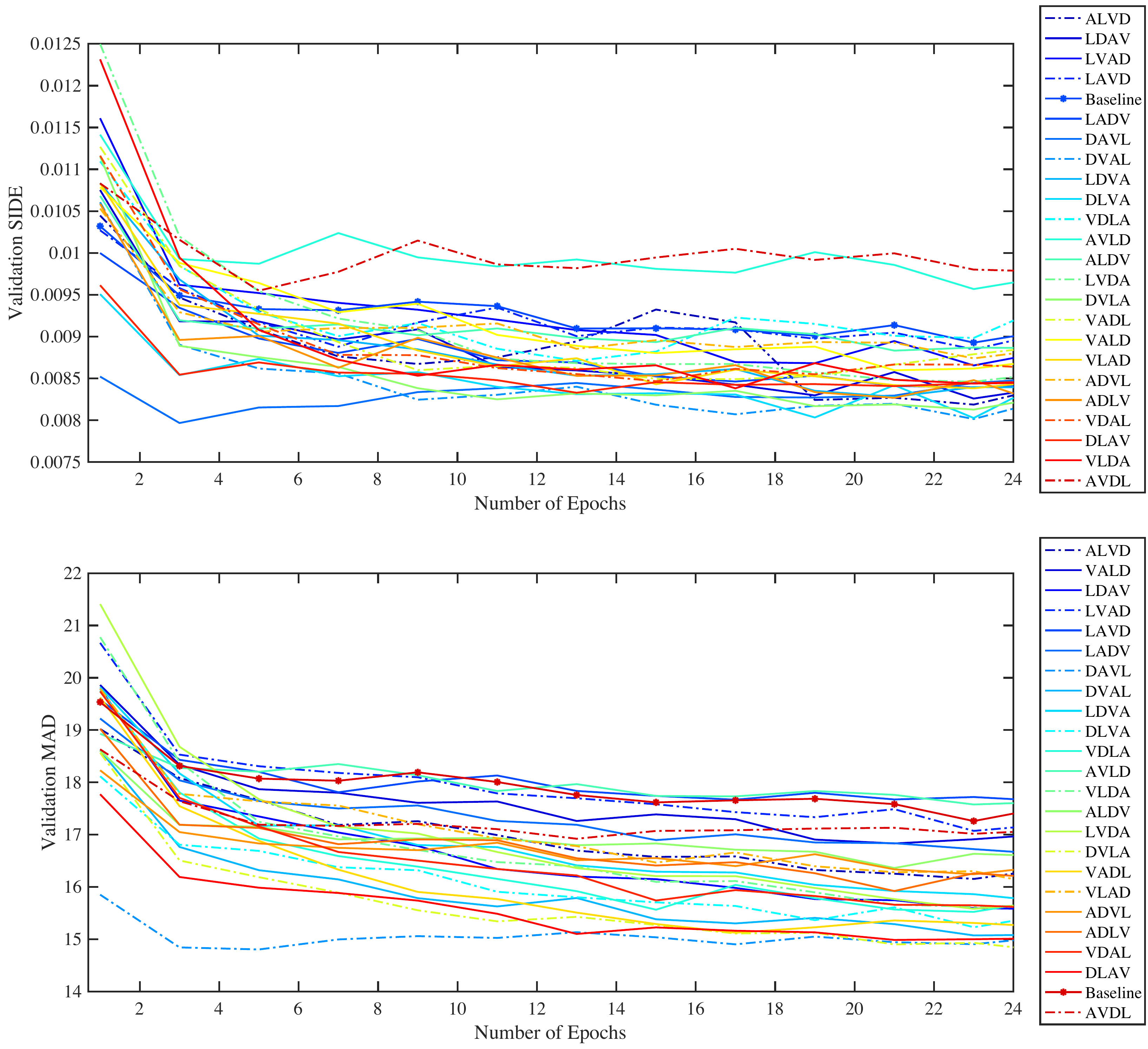}
  \caption{\footnotesize Full results of SIDE (top) and MAD (bottom) on the BFM dataset~\cite{paysan20093d} for different dense causal orderings. All the results are averaged over three different random seeds. For those dense orderings that are not plotted, we leave them out in the figure because they can not converge properly.}\label{dense_order_full}
\end{figure}

The experimental settings generally follow Appendix~\ref{exp_detail}. We plot all the dense orderings that properly converge on the BFM dataset in Fig.~\ref{dense_order_full}. Interestingly, we find that there is one dense ordering (\ie, depth-albedo-lighting-viewpoint) that is always difficult to converge. It fails to converge with all the three random seeds, showing that this dense ordering may disobey some underlying causality. The results also further justify that different dense orderings yield consistently different inductive biases and hence different out-of-distribution robustness/generalizability.

\clearpage
\newpage
\section{Formal Justification from Function Approximation}\label{app_formal_func}

\begin{theorem}[Advantages of Causal Ordering]\label{thm1}
Let the underlying lighting function be the power series $\bm{L}:=g^*(\bm{I})=\sum_{\bm{k}\in \mathbb{N}^d} a_{\bm{k}} \bm{I}^{\bm{k}}=\sum_{\bm{k}\in \mathbb{N}^d} a_{\bm{k}} I_1^{k_1}I_2^{k_2}\cdots I_d^{k_d}$ which is absolutely convergent in $[-1,1]^d$. Assume the encoders $f_L,f_V$ and $h_L$ are ReLU neural networks with depth $T$ and width $d+4$. Then for any $\delta>0$ and all $\bm{I}$, there exists a lighting function $\bm{L}_a=f_L(\bm{I})$ (Fig.~\ref{example2}(a)) that reaches the approximation accuracy: $|f_L(\bm{I})-g^*(\bm{I})|<2\sum_{\bm{k}\in\mathbb{N}^d}|a_{\bm{k}}|\exp(-d\delta(e^{-1}T^{\frac{1}{2d}}-1))$. In contrast, there exists a lighting function $\bm{L}_b=f_L(\bm{I})+h_L\circ f_V(\bm{I})$ (Fig.~\ref{example2}(b)) achieving $|f_L(\bm{I})+h_L\circ f_V(\bm{I})-g^*(\bm{I})|<2\sum_{\bm{k}\in\mathbb{N}^d}|a_{\bm{k}}|\exp(-d\delta(e^{-1}(3T)^{\frac{1}{2d}}-1))$. By construction, $\bm{L}_b$ can always achieve better approximation than $\bm{L}_a$. 
\end{theorem}

\begin{proof}

In order to prove the theorem, the basic proof sketch is as follows:
\begin{itemize}
    \item We first reduce addition and composition of multiple neural networks to a single neural network with increasing depth;
    \item We then derive a convergence rate where the depth improves the approximation error exponentially.
\end{itemize}

Before we introduce the detailed proof, we briefly review the problem setup.

\begin{definition}
Given a function $f(x_1,\cdots,x_d)$, if there are variables $\{y_{1:T,1:K}\}$ where $\{y_{1:T,1:K}\}=\{x_{i,j}|i=1,\cdots,T;j=1,\cdots,K\}$ such that we have
\begin{equation}\label{thm_defn1}
\begin{aligned}
y_{1,k}&\in\textnormal{RLinear}(x_{1:d})\\
y_{t+1,k}&\in\textnormal{RLinear}(x_{1:d},y_{t,1:T})\\
f&\in\textnormal{Linear}(x_{1:d},y_{1:T,1:K})
\end{aligned}
\end{equation}
where $k=1,\cdots,K$, $t=1,\cdots,T$, $\textnormal{Linear}(x_{1:d})$ denotes the set of arbitrary linear combinations of $x_{1:d}$ (\ie, there exist $\beta_i\in\mathbb{R}$ such that $y=\beta_0+\beta_1 x_1+\cdots+\beta_d x_d$) and $\textnormal{RLinear}(x_{1:d})$ denotes the set of arbitrary linear combinations with ReLU activation (\ie, $\textnormal{RLinear}(x_{1:d})=\textnormal{ReLU}(\textnormal{Linear}(x_{1:d}))=\max(\textnormal{Linear}(x_{1:d}),0)$), then $f$ is said to be in the neural network class $\mathcal{F}_{T,K}(\mathbb{R}^d)$ and $\{y_{1:T,1:K}\}$ is a set of hidden variables of $f$.
\end{definition}

In fact, the neural networks considered here is slightly different from the standard ResNet~\cite{he2016deep}. We are using neural networks with skip connections from the input layer to the hidden layers and from the hidden layers to the output layer. However, such neural networks are equivalent to standard fully connected neural networks without skip connection.

\begin{proposition}
A function $f\in \mathcal{F}_{T,K}(\mathbb{R}^d)$ can be represented by a ReLU network with depth $T+1$ and width $K+d+1$.
\end{proposition}
\begin{proof}
We first require the hidden variables of $f$ (\ie, $\{y_{1:T,1:K}\}$) to satisfy Eq.~\eqref{thm_defn1}, namely
\begin{equation}
    f=\alpha_0 + \sum_{i=1}^d \alpha_i x_i +\sum_{t=1}^T \sum_{k=1}^K \beta_{t,k} y_{t,k}.
\end{equation}
Then we construct the following variables $\{h_{1:T,1:K}\}$:
\begin{equation}
\begin{aligned}
    h_{t,1:K}&=y_{t,1:K}\\
    h_{t,K+1:K+d}&=x_{1:d}
\end{aligned}
\end{equation}
where $t=1,\cdots,T$, and additionally,
\begin{equation}
\begin{aligned}
h_{1,K+d+1}&=\alpha_0 + \sum_{i=1}^d\alpha_i x_i\\
h_{t+1,K+d+1}&=h_{t,K+d+1}+\sum_{k=1}^K\beta_{t,k}h_{t,k}
\end{aligned}
\end{equation}
where $t=1,\cdots,T-1$. Because $h_{1,k}\in\textnormal{RLinear}(x_{1:d})$, $h_{t+1,k}\in\textnormal{RLinear}(h_{t,1:K+d+1})$ for $k=1,\cdots,K+d+1$ and  $t=1,\cdots,T-1$, we can observe that $f\in\textnormal{Linear}(h_{T,1:K+d+1})$ is a representation of a standard neural network. $\hfill \blacksquare$
\end{proof}

We discuss in the next two propositions how addition and composition of neural networks can be viewed as increasing the number of depth.

\begin{proposition}\label{func_add}
For the addition of two neural networks $f_1$ and $f_2$, if $f_1\in\mathcal{F}_{T_1,K}(\mathbb{R}^d)$ and $f_2\in\mathcal{F}_{T_2,K}(\mathbb{R}^d)$, then $f_1+f_2\in\mathcal{F}_{T_1+T_2,K}$. This also leads to
\begin{equation}
    \mathcal{F}_{T_1,K}+\mathcal{F}_{T_2,K}\subseteq \mathcal{F}_{T_1+T_2,K}
\end{equation}
which indicates that the addition of two neural networks of the same width is equivalent to a single neural network with the same width and the added depth of the two neural networks.
\end{proposition}
\begin{proof}
We define $\{y^1_{1:T_1,1:K}\}$ and $\{y^2_{1:T_2,1:K}\}$ as the hidden variables of $f_1$ and $f_2$, respectively. Then we let
\begin{equation}
\begin{aligned}
y_{1:T_1,1:K}&=y^1_{1:T_1,1:K}\\
y_{T_1+1:T_1+T_2,1:K}&=y^2_{1:T_2,1:K}.
\end{aligned}
\end{equation}
Therefore, we have that $\{y_{1:T_1+T_2,1:K}\}$ is a set of hidden variables for $f_1+f_2$, leading to $f_1+f_2\in\mathcal{F}_{T_1+T_2,K}$.$\hfill \blacksquare$
\end{proof}

\begin{proposition}\label{func_composition}
For the composition of two neural networks $f_1$ and $f_2$, if we have $f_1(x_1,\cdots,x_d)\in\mathcal{F}_{T_1,K+1}(\mathbb{R}^d)$ and $f_2(y,x_1,\cdots,x_d)\in\mathcal{F}_{T_2,K}(\mathbb{R}^{d+1})$, then $f_2(f_1(x_1,\cdots,x_d),x_1,\cdots,x_d)\in\mathcal{F}_{T_1+T_2,K+1}(\mathbb{R}^d)$. This also leads to
\begin{equation}
    \mathcal{F}_{T_2,K}\circ\mathcal{F}_{T_1,K+1}\subseteq \mathcal{F}_{T_1+T_2,K+1}
\end{equation}
which indicates that the composition of two neural networks can be roughly viewed as a single neural network with added depth of these two neural networks.
\end{proposition}
\begin{proof}
We define $\{y^1_{1:T_1,1:K}\}$ and $\{y^2_{1:T_2,1:K}\}$ as the hidden variables of $f_1$ and $f_2$, respectively. Then we let
\begin{equation}
    \begin{aligned}
    y_{1:T_1,1:K+1}&=y^1_{1:T_1,1:K+1}\\
    y_{T_1+1:T_1+T_2,1:K}&=y^2_{1:T_2,1:K}\\
    y_{T_1+1,L+1}=y_{T_1+2,K+1}&=\cdots=y_{T_1+T_2,K+1}=f_1(x_1,\cdots,x_d)
    \end{aligned}
\end{equation}
Since $\{y_{1:T_1+T_2,1:K+1}\}$ is a set of hidden variables of $f_2(f_1(x_1,\cdots,x_d),x_1,\cdots,x_d)$, then we can see that the composition property holds. $\hfill \blacksquare$
\end{proof}

Then we introduce the following lemma to establish the connection between depth and approximation error of analytic functions:
\begin{lemma}[Simplified results from \cite{wang2018exponential}]\label{thm_lemma}
Let $f$ be an analytic function over $(-1,1)^d$. Assume that the power series $f(\bm{x})=\sum_{\bm{i}\in\mathbb{N}^d}a_{\bm{i}}\bm{x}^{\bm{i}}$ is absolutely convergent in $[-1,1]^d$, where $\bm{x}=[x_1,\cdots,x_d]$. Then for any $\delta>0$, there exists a function $\hat{f}$ that can be represented by a deep ReLU neural network with depth $T$ and width $d+4$, such that 
\begin{equation}
    \left| f(\bm{x})-\hat{f}(\bm{x}) \right|<2\sum_{\bm{i}\in\mathbb{N}^d}\cdot\exp\left( -d\delta\left(e^{-1}T^{\frac{1}{2d}}-1\right) \right)
\end{equation}
which holds for any $\bm{x}\in [-1+\delta,1-\delta]^d$.
\end{lemma}

Suppose $f_L,h_L,f_V$ are ReLU neural networks with depth $T$ and width $d+4$. We consider the discrepancy between standard lighting function from independent SCR: $f_L(\bm{I})$ and underlying lighting function $\bm{L}:=g^*(\bm{I})=\sum_{\bm{k}\in \mathbb{N}^d} a_{\bm{k}} \bm{I}^{\bm{k}}$. Applying Lemma~\ref{thm_lemma}, we end up with 
\begin{equation}
    \left|f_L(\bm{I})-g^*(\bm{I})\right|<2\sum_{\bm{k}\in\mathbb{N}^d}|a_{\bm{k}}|\exp\left(-d\delta\left(e^{-1}T^{\frac{1}{2d}}-1\right)\right).
\end{equation}

According to the proof of Lemma~\ref{thm_lemma} in \cite{wang2018exponential}, we learn that $\hat{f}$ is a function from $\mathcal{F}_{T,3}$. We can use $f_L$ and $f_V$ to denote functions in $\mathcal{F}_{T,3}$. Then we use $h_L$ to denote functions from $\mathcal{F}_{T,2}$. This can be easily done by setting some of the neurons to be zero. In fact, we only need to let $h_L$ to be a neural network with width $d+3$, which can further weaken the current assumption. Then based on Proposition~\ref{func_add} and Proposition~\ref{func_composition}, we obtain that the lighting function $f_L(\bm{I})+h_L\circ f_V(\bm{I})$ can represent arbitrary function in $\mathcal{F}_{3T,3}$. Finally, we apply Lemma~\ref{thm_lemma} again and show that 
\begin{equation}
    \left|f_L(\bm{I})+h_L\circ f_V(\bm{I})-g^*(\bm{I})\right|<2\sum_{\bm{k}\in\mathbb{N}^d}|a_{\bm{k}}|\exp\left(-d\delta\left(e^{-1}(3T)^{\frac{3}{2d}}-1\right)\right)
\end{equation}
which achieves better convergence for the approximation than the lighting function from independent SCR. This concludes the proof. $\hfill \blacksquare$
\end{proof}

\newpage
\section{Experiments on CUDA-9}\label{exp_cuda9}
All the experiments in the main paper are run on CUDA-10, so the experimental settings are fair. However, in order to make comprehensive comparison, we also run our methods on CUDA-9 which is exactly the same as \cite{wu2020unsupervised}. We put the quantitative results in Table~\ref{bfm_cuda9}. %In fact, the performance mismatch on different CUDA versions have been discovered simultaneously by some other GitHub users\footnote{Please see the following issues in the official Github repository of \cite{wu2020unsupervised}: \url{https://github.com/elliottwu/unsup3d/issues/15} and \url{https://github.com/elliottwu/unsup3d/issues/12}.}.
\begin{table}[h]
    \footnotesize
	\centering
	\renewcommand{\captionlabelfont}{\footnotesize}
	\vspace{-3mm}
	\begin{tabular}{c|cc}
		%\hline
		  Method & SIDE ($\times 10^{-2}$) $\downarrow$ & MAD (deg.) $\downarrow$\\
		  \shline
		   Supervised & \textbf{0.410} {\scriptsize$\pm$0.103} & \textbf{10.78} {\scriptsize$\pm$1.01} \\
           Constant Null Depth & 2.723 {\scriptsize$\pm$0.371} &  43.34 {\scriptsize$\pm$2.25} \\
           Average GT Depth &  1.990 {\scriptsize$\pm$0.556} &  23.26 {\scriptsize$\pm$2.85} \\\hline
           Wu et al.~\cite{wu2020unsupervised} (reported) & \textbf{0.793} {\scriptsize$\pm$0.140} & 16.51 {\scriptsize$\pm$1.56} \\
           Ho et al.~\cite{ho2021toward} (reported) & 0.834 {\scriptsize$\pm$0.169} & \textbf{15.49} {\scriptsize$\pm$1.50} \\
           Wu et al.~\cite{wu2020unsupervised} (our run) & 0.791 {\scriptsize$\pm$0.143} & 16.35 {\scriptsize$\pm$1.55} \\\hline\rowcolor{Gray}
           Independent SCR & 0.795 {\scriptsize$\pm$0.141} & 16.26 {\scriptsize$\pm$1.58}  \\\rowcolor{Gray}
           Dense SCR (BO) & 0.693 {\scriptsize$\pm$0.153} & 13.30 {\scriptsize$\pm$1.83}\\\rowcolor{Gray}
           Generic SCR (Eq.~\ref{unroll_obj_dense}) & \textbf{0.687} {\scriptsize$\pm$0.172} & \textbf{13.22} {\scriptsize$\pm$1.90} \\\rowcolor{Gray}
           Dynamic SCR  & 0.690 {\scriptsize$\pm$0.165} & 13.27 {\scriptsize$\pm$1.87} 
	\end{tabular}
	\vspace{3mm}
	\caption{\footnotesize Depth reconstruction results on BFM (CUDA-9).} \label{bfm_cuda9}
\vspace{-10mm}
\end{table}

\begin{figure}[h]
  \renewcommand{\captionlabelfont}{\footnotesize}
  \setlength{\abovecaptionskip}{8pt}
  \setlength{\belowcaptionskip}{8pt}
  \centering
  \vspace{-5mm}
  \includegraphics[width=4.6in]{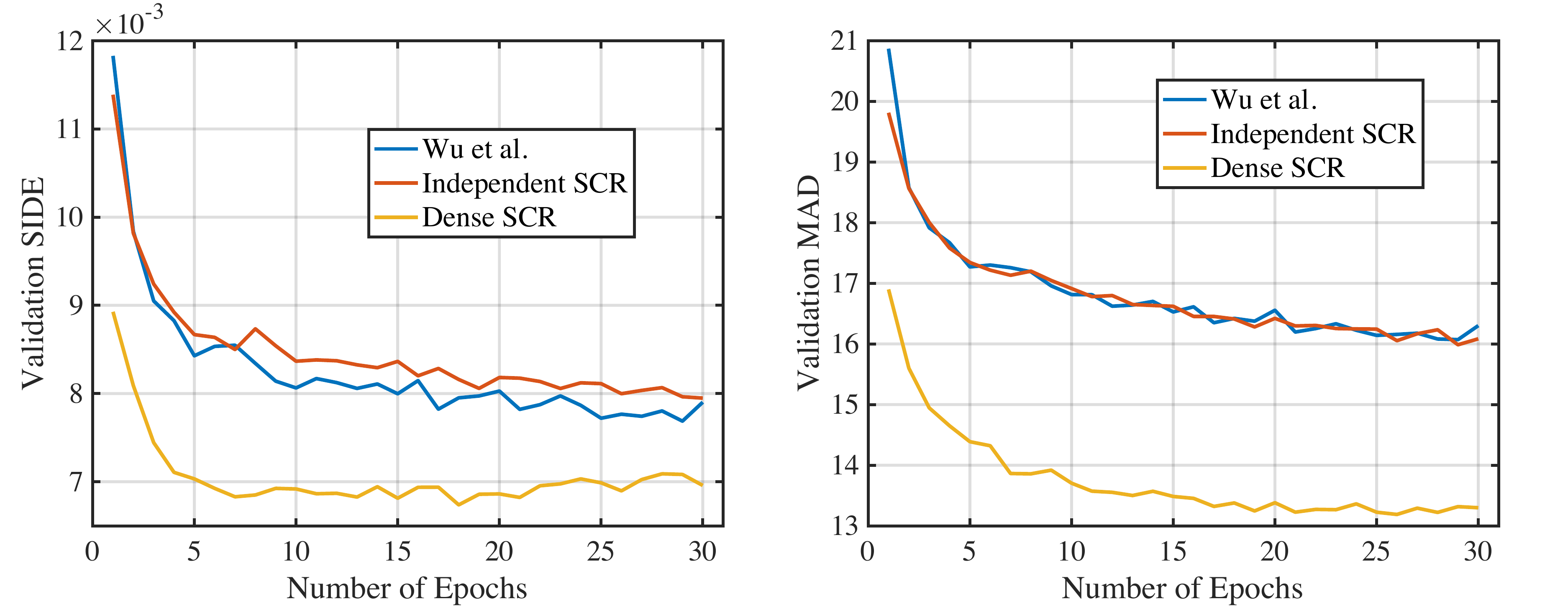}
  \caption{\footnotesize Left: SIDE on validation set; Right: MAD on validation set. (CUDA-9)}\label{convergence_cuda9}
\end{figure}

From Table~\ref{bfm_cuda9}, we observe a significant performance boost from dense SCR using BO compared to the result on CUDA-10. The final learned ordering is depth-albedo-viewpoint-lighting. On CUDA-9, the baseline~\cite{wu2020unsupervised} we run is also able to match the reported performance. Again, we verify that our own baseline (independent SCR) performs similarly to \cite{wu2020unsupervised}. For the other SCR variants, we also observe similar performance gain. Since these results do not affect the conclusion drawn in the main paper, we omit them here.

In order to better compare dense SCR to independent SCR and \cite{wu2020unsupervised}, we also plot the convergence curve for the validation SIDE and MAD. The convergence curves in Table~\ref{convergence_cuda9} show similar pattern to the ones on CUDA-10 (Fig.~\ref{convergence}). Dense SCR shows exceptional convergence speed compared to both independent SCR and \cite{wu2020unsupervised}, partially supporting our hypothesis that there exist some dense ordering that matches the underlying causal ordering and is able to perform fast and disentangled reconstruction. In general, we find that re-running the experiments on CUDA-9 only amplifies our performance gain, which better validates the superiority of SCR.

\newpage
\section{Visualization of Learned Orderings}

We typically learn SCR on the BFM dataset and use the exactly same causal ordering on the other datasets such as CelebA and cat faces, because we believe there exist common causality when reconstructing human or animal faces. Our experimental results show that this is indeed the case and the learned ordering can be transferred to other similar datasets. Therefore, we visualize some of the learned SCR on the BFM datasdet.
\vspace{-2mm}
\subsection{Learned Orderings of Dense SCR}
\vspace{-1mm}
\begin{figure}[h]
  \renewcommand{\captionlabelfont}{\footnotesize}
  \setlength{\abovecaptionskip}{7pt}
  \setlength{\belowcaptionskip}{-5pt}
  \centering
  \vspace{-5mm}
  \includegraphics[width=4.6in]{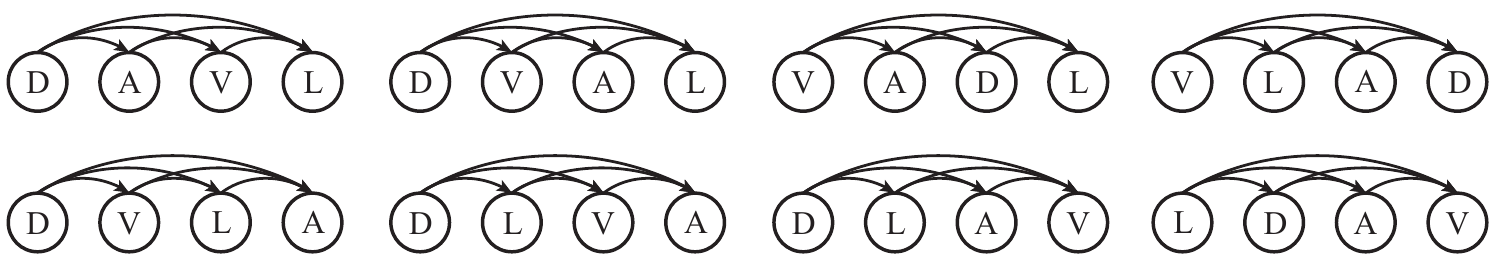}
  \caption{\footnotesize Some learned dense causal orderings using Bayesian optimization.}\label{dense_lfo_vis}
\end{figure}

We visualize some of the learned dense orderings using Bayesian optimization in Fig.~\ref{dense_lfo_vis}. We can observe that these learned dense orderings are generally quite similar. Some of them only differ by one pair-wise permutation, such as DAVL and DVAL. Moreover, we observe that viewpoint is usually put in front of lighting and depth is put in front of lighting as well. In general, these dense orderings well cover the potential arrangements done by domain experts.

\vspace{-2mm}
\subsection{Learned Orderings of Generic SCR}
\vspace{-1mm}
\begin{figure}[h]
  \renewcommand{\captionlabelfont}{\footnotesize}
  \setlength{\abovecaptionskip}{7pt}
  \setlength{\belowcaptionskip}{-5pt}
  \centering
  \vspace{-5mm}
  \includegraphics[width=4.6in]{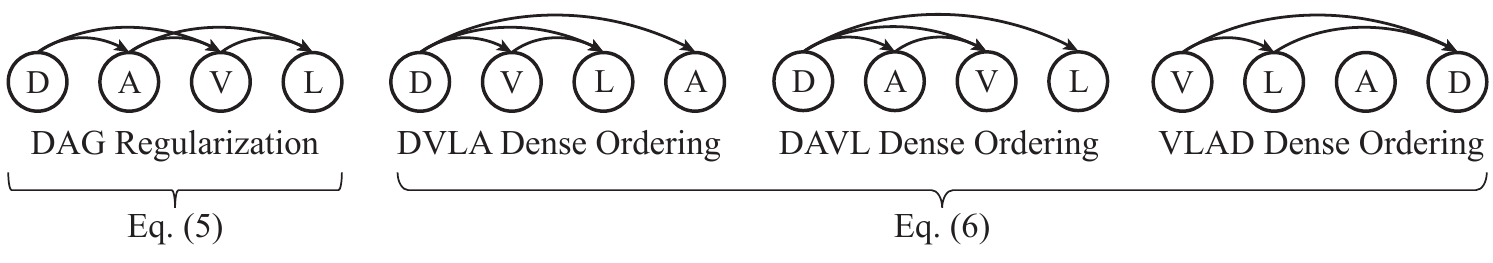}
  \caption{\footnotesize Some learned generic causal orderings using either DAG regularization (Eq.~\eqref{unroll_obj}) or dense ordering mask (Eq.~\eqref{unroll_obj_dense}).}\label{generic_lfo_vis}
\end{figure}

We visualize the learned generic orderings using either Eq.~\eqref{unroll_obj} or Eq.~\eqref{unroll_obj_dense} in Fig.~\ref{generic_lfo_vis}. We find that learning DAGs with a dense ordering mask usually leads to more sparse DAGs, compared to the general DAG regularization. It may be because the dense ordering has largely reduced the feasible space of the resulting DAG. However, if we start with a dense ordering that yields poor generalizability, then it is unlikely that generic SCR can learn a well-performing DAG. Specifically, VLAD returns a DAG with fewer edges while DVLA and DAVL return DAGs with more edges. We suspect that because the dense orderings DVLA and DAVL perform better than VLAD by a considerable margin, both DVLA and DAVL may contain more crucial causal directions than VLAD. Therefore, generic SCR tends to keep more edges for DVLA and DAVL while drop more edges for VLAD. Interesting, if generic SCR drops all the edges, we will end up with an independent SCR. If generic SCR learns to drop all the edges, it is likely that the initial dense ordering performs poorly. In contrast, learning generic SCR with Eq.~\eqref{unroll_obj} requires less prior knowledge and is not dependent on the performance of the initial dense ordering.

\vspace{-2mm}
\subsection{Learned Orderings of Dynamic SCR}
\vspace{-1mm}
\begin{figure}[h]
  \renewcommand{\captionlabelfont}{\footnotesize}
  \setlength{\abovecaptionskip}{7pt}
  \setlength{\belowcaptionskip}{-5pt}
  \centering
  \vspace{-5mm}
  \includegraphics[width=4.6in]{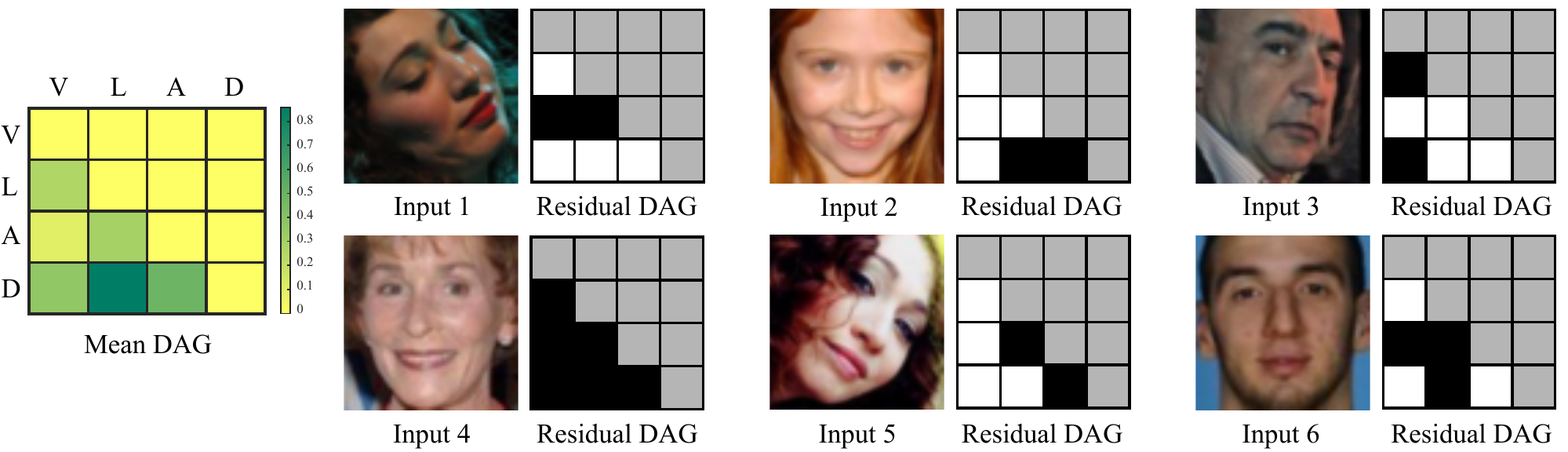}
  \caption{\footnotesize Some examples of the learned dynamic causal orderings. For the DAG adjacency matrix $\bm{M}$, each element $M_{i,j}$ is defined as the directed edge from $j$ to $i$. The mean DAG is the average adjacency matrix over all the training samples. In the residual DAG, the white block denotes this edge weight is decreased (from the mean DAG) for the input, the black block indicates this edge weight is increased (from the mean DAG) for the input, and the gray block has been masked out by the initial dense ordering.}\label{dyn_lfo_vis}
\end{figure}

We visualize some learned dynamic orderings for dynamic SCR in Fig.~\ref{dyn_lfo_vis}. In order to give a more intuitive visualization, we randomly select 6 input images and use the learned dynamic SCR to infer its continuous DAG (\ie, the edge weight is continuous from $0$ to $1$ because of cosine similarity) and compare their difference. For Input 1, we find that the albedo requires additional constraints from viewpoint and lighting. This makes intuitive senses, since this image has poor lighting and it may be difficult to disentangle its lighting and albedo. Moreover, its viewpoint is also challenging to estimate. For Input 4, this image is very blurry, making all the 3D factors difficult to estimate. Therefore, dynamic SCR tends to push the DAG to be the dense ordering such that different factors can pose constraints to each other, leading to better disentanglement.

\newpage
\section{Additional Experiments}
\vspace{-1mm}
\subsection{Pose Preservation between Input and Reconstructed Faces}
\vspace{-0.5mm}

We use a face pose pretrained network from \cite{albiero2021img2pose} to estimate the pose of the original and reconstructed faces, and the compare their angle difference. The results are shown in Table~\ref{pose}. We can observe that dense SCR, generic SCR and dynamic SCR preserves better pose for the reconstructed faces.

\begin{table}[h]
\footnotesize
\centering
\renewcommand{\captionlabelfont}{\footnotesize}
\vspace{-3mm}
\begin{tabular}{c | c} 
\specialrule{0em}{-6pt}{0pt}
Method & Angle Difference $\downarrow$ \\
\shline
Wu et al.~\cite{wu2020unsupervised} & 12.5 $\pm$ 8.9 \\
Independent SCR & 12.3 $\pm$ 9.3 \\\hline\rowcolor{Gray}
Dense SCR & 11.0 $\pm$ 9.1 \\\rowcolor{Gray}
Generic SCR & 10.8 $\pm$ 8.3  \\\rowcolor{Gray}
Dynamic SCR & \textbf{10.5 $\pm$ 8.2}  \\
\end{tabular}
\vspace{1mm}
\caption{\footnotesize Angle difference (degree) between input and reconstructed faces.}
\vspace{-4mm}
\label{pose}
\end{table}

\vspace{-10mm}
\subsection{Identity Preservation between Input and Reconstructed Faces}
\vspace{-0.5mm}

We use three views for the reconstructed image (see Fig.~\ref{id_exp}). For each test image on CelebA, we have 4 positive samples (original + 3 reconstructed images) and the images from different identities are negative samples. We randomly construct 3,000 positive pairs and 3,000 negative pairs as the testing set. We use a pretrained FaceNet~\cite{schroff2015facenet} to compute the cosine similarity between positive and negative pair. Results in Table~\ref{id_test} show that all the SCR variants yield better identity preservation then both \cite{wu2020unsupervised} and independent SCR. We also use SphereFace\cite{liu2017sphereface,liu2022sphereface} and conclude the same advantage for the proposed SCR.

\begin{figure}[h]
  \renewcommand{\captionlabelfont}{\footnotesize}
  \setlength{\abovecaptionskip}{2.5pt}
  \setlength{\belowcaptionskip}{-14pt}
  \centering
  \vspace{-4.5mm}
  \includegraphics[width=3in]{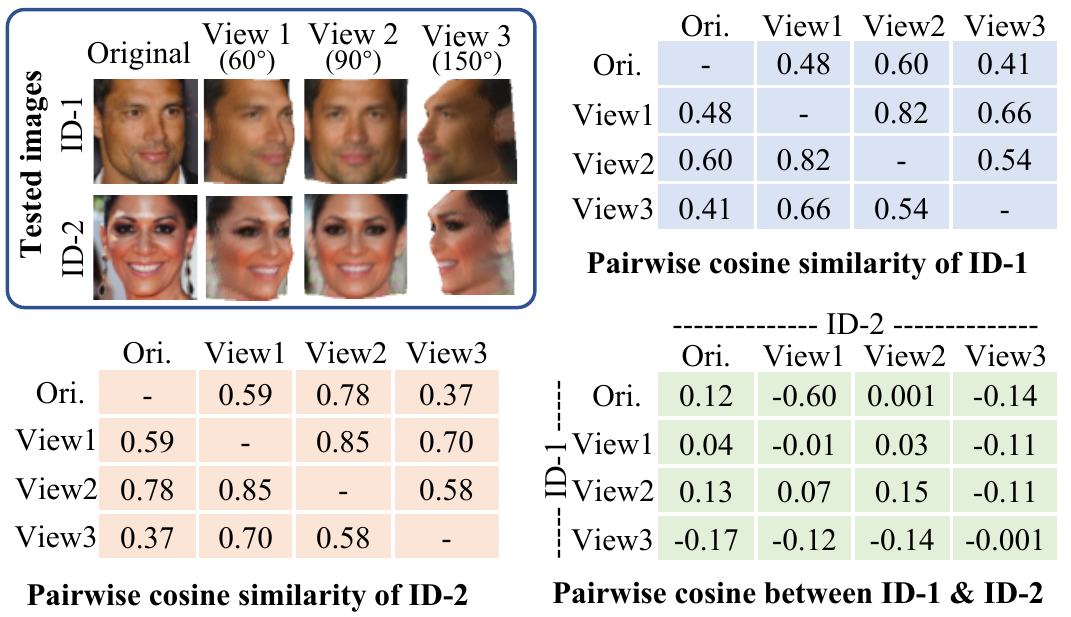}
  \caption{\footnotesize An example of identity preservation.}\label{id_exp}
\end{figure}

\begin{table}[h]
\footnotesize
\centering
\renewcommand{\captionlabelfont}{\footnotesize}
\begin{tabular}{c | c c} 
\specialrule{0em}{-18pt}{0pt}
  Method & ~~Avg. Pos. Cos. Sim. $\uparrow$~~ & ~~Avg. Neg. Cos. Sim. $\downarrow$~~ \\
 \shline
Wu et al.~\cite{wu2020unsupervised} & 0.63 $\pm$ 0.25 & 0.18 $\pm$ 0.26 \\
Independent SCR & 0.64 $\pm$ 0.22 & 0.19 $\pm$ 0.24 \\\hline\rowcolor{Gray}
Dense SCR & 0.68 $\pm$ 0.19 & 0.14 $\pm$ 0.19  \\\rowcolor{Gray}
Generic SCR & 0.67 $\pm$ 0.17 & \textbf{0.12 $\pm$ 0.15}  \\\rowcolor{Gray}
Dynamic SCR & \textbf{0.68 $\pm$ 0.18} & 0.12 $\pm$ 0.18   \\
\end{tabular}
\vspace{1mm}
\caption{\footnotesize Average cosine similarity of positive/negative pairs.}
\vspace{-5mm}
\label{id_test}
\end{table}

\clearpage
\newpage
\section{More Qualitative Results}

\subsection{Synthetic Faces}
\begin{figure}[h]
  \renewcommand{\captionlabelfont}{\footnotesize}
  \setlength{\abovecaptionskip}{7pt}
  \setlength{\belowcaptionskip}{-5pt}
  \centering
  \vspace{-5mm}
  \includegraphics[width=4.7in]{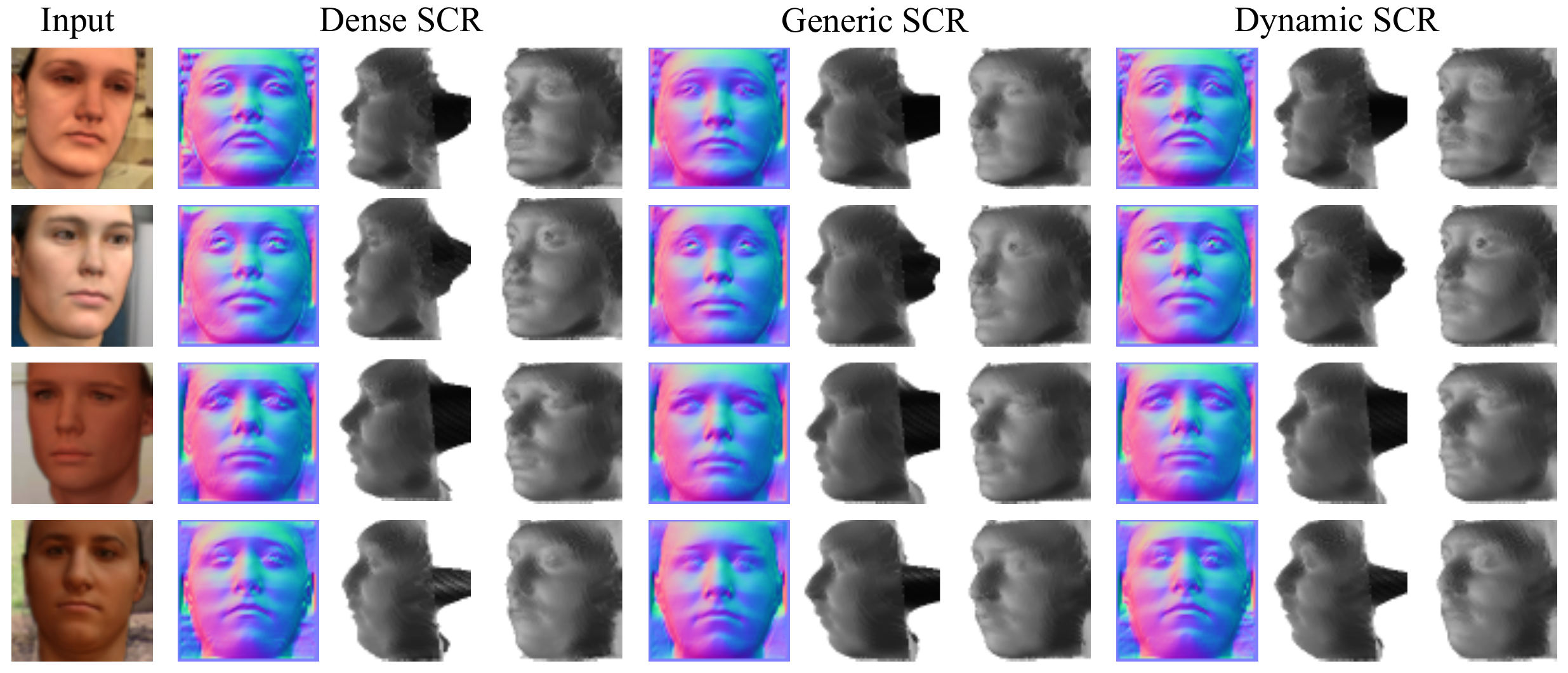}
  \caption{\footnotesize Qualitative results of normal and shapes on synthetic faces.}\label{sface_vis_app}
\end{figure}
\subsection{CelebA}
\begin{figure}[h]
  \renewcommand{\captionlabelfont}{\footnotesize}
  \setlength{\abovecaptionskip}{7pt}
  \setlength{\belowcaptionskip}{-15pt}
  \centering
  \vspace{-5mm}
  \includegraphics[width=4.7in]{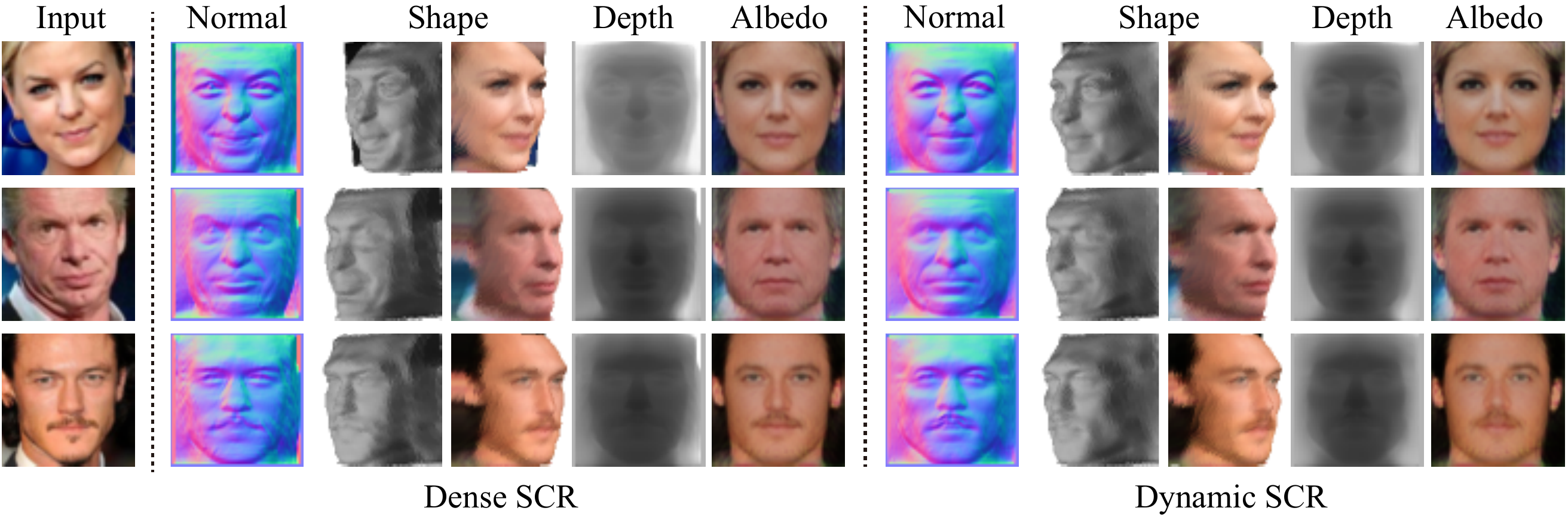}
  \caption{\footnotesize Qualitative results on CelebA faces.}\label{cface_vis_app}
\end{figure}
\newpage
\vspace{-2mm}
\subsection{Cat Faces}
\vspace{-1mm}
\begin{figure}[h]
  \renewcommand{\captionlabelfont}{\footnotesize}
  \setlength{\abovecaptionskip}{7pt}
  \setlength{\belowcaptionskip}{-12pt}
  \centering
  \vspace{-6mm}
  \includegraphics[width=4.7in]{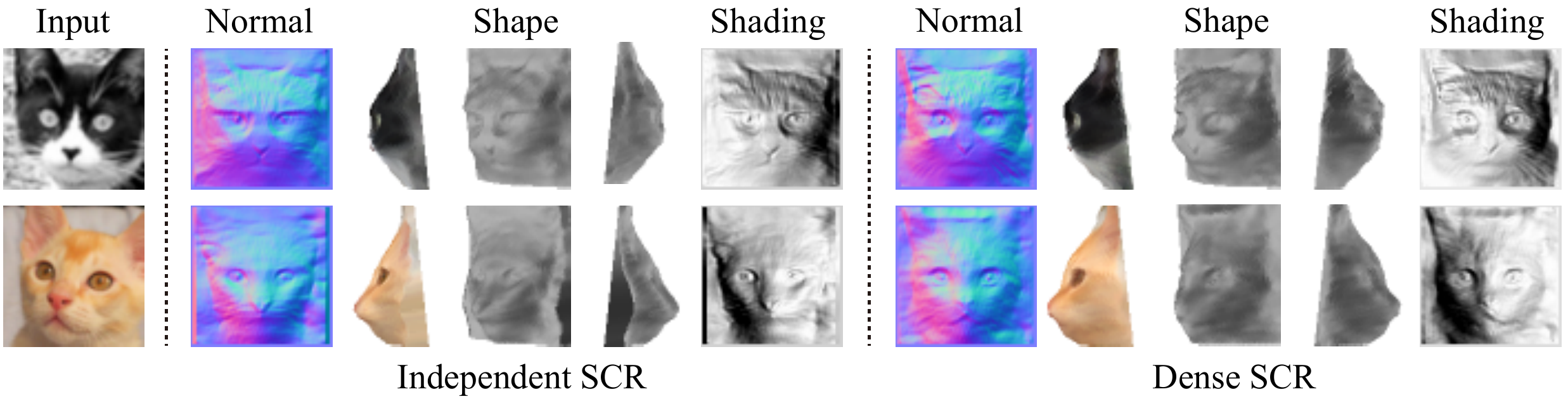}
  \caption{\footnotesize Qualitative results on cat faces.}\label{cat_vis_app}
\end{figure}
\vspace{-2mm}
\subsection{Cars}
\vspace{-1mm}
\begin{figure}[h]
  \renewcommand{\captionlabelfont}{\footnotesize}
  \setlength{\abovecaptionskip}{7pt}
  \setlength{\belowcaptionskip}{-12pt}
  \centering
  \vspace{-6mm}
  \includegraphics[width=4in]{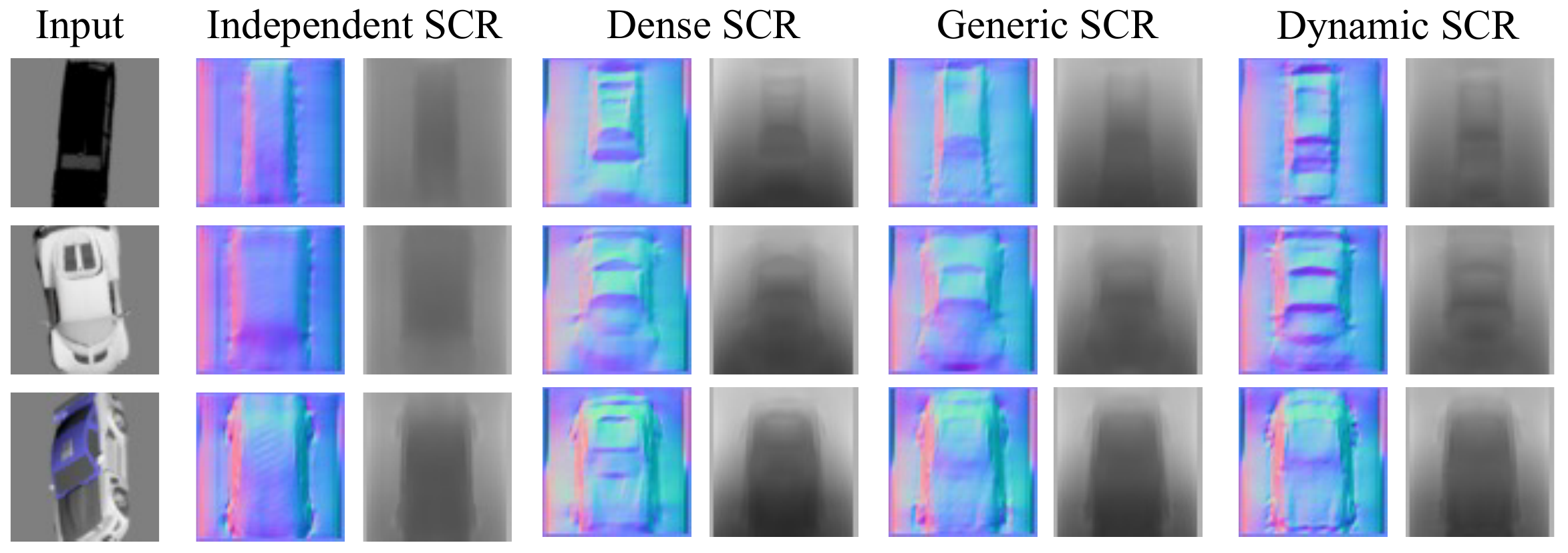}
  \caption{\footnotesize Qualitative results of normal and depth on cars.}\label{car_vis_app}
\end{figure}

\vspace{-2mm}
\subsection{Comparison of Different Datasets}
\vspace{-1mm}
\begin{figure}[h]
  \renewcommand{\captionlabelfont}{\footnotesize}
  \setlength{\abovecaptionskip}{7pt}
  \setlength{\belowcaptionskip}{-12pt}
  \centering
  \vspace{-6mm}
  \includegraphics[width=4in]{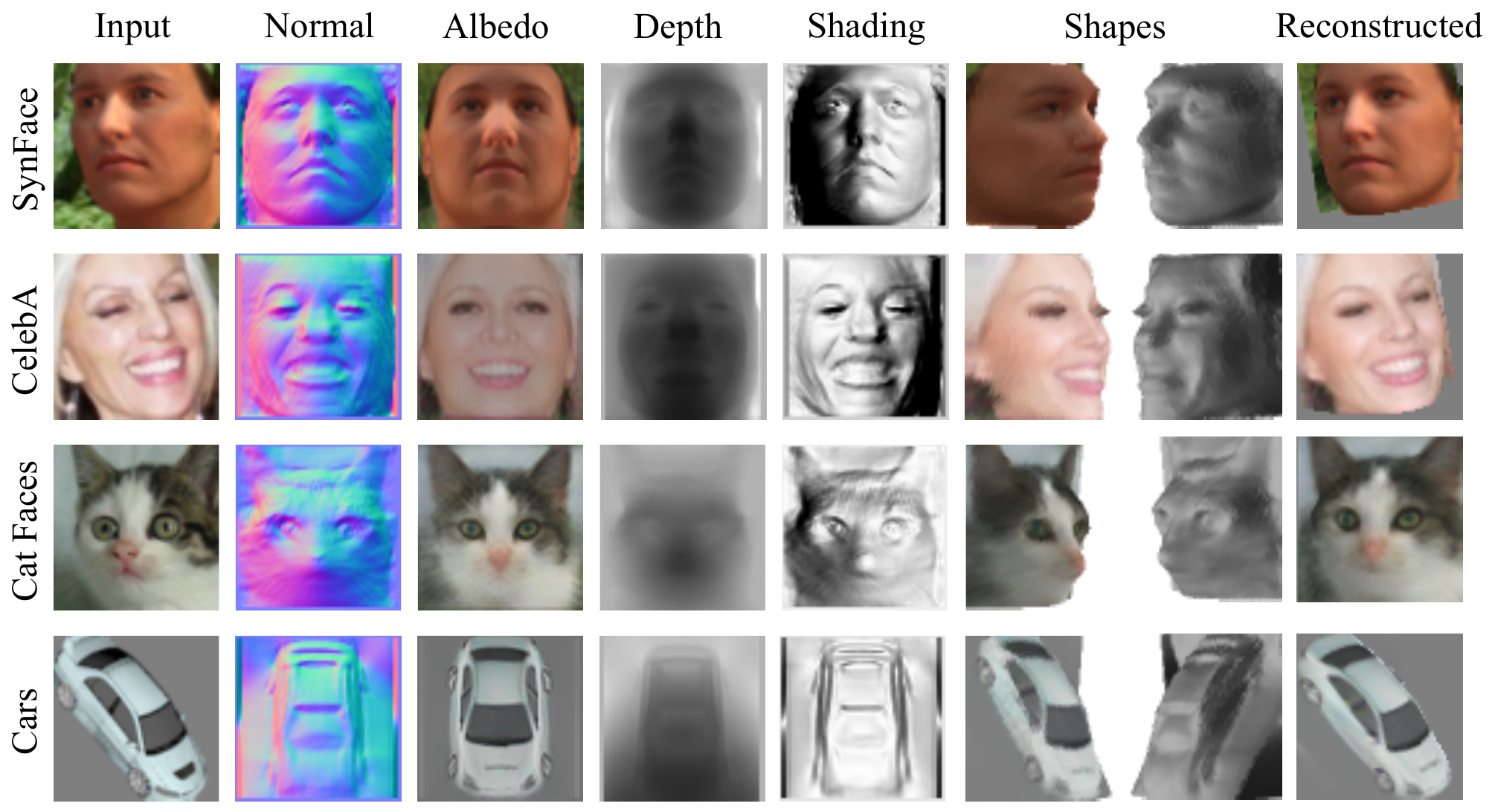}
  \caption{\footnotesize Qualitative results of all the 3D factors on different datasets (Dense SCR).}\label{all_qual}
\end{figure}

%\end{appendix}
\end{document}